%% file: wMixture-arxiv1708.00568V3.tex
\documentclass[11pt]{article}
\usepackage{fullpage,url,graphicx,amssymb,amsmath,cite}
\usepackage[makeroom]{cancel}
\newtheorem{proposition}{Proposition}
 
\def\calP{\mathcal{P}}
\def\tm{\tilde{m}}
\def\calE{\mathcal{E}}
\def\tw{\tilde{w}}
\def\tp{\tilde{p}}
\def\tq{\tilde{q}}
\def\eqdef{\mathrm{:=}}

\def\ML{\mathrm{ML}}
\def\MLE{\mathrm{MLE}}
\def\KL{\mathrm{KL}}
\def\IS{\mathrm{IS}}

\def\dnu{\mathrm{d}\nu}
\def\dmu{\mathrm{d}\mu}
\def\dx{\mathrm{d}x}

\def\calN{\mathcal{N}}
\def\calX{\mathcal{X}}
\def\calY{\mathcal{Y}}

\def\calM{\mathcal{M}}

\def\calO{\mathcal{O}}
\def\calA{\mathcal{A}}

\def\bbR{\mathbb{R}}
\def\inner#1#2{ \langle {#1},{#2} \rangle }
\newenvironment{proof}{\noindent {Proof:}}{\hfill$\square$}
\def\Bhat{\mathrm{Bhat}}

\def\TV{\mathrm{TV}}
 \def\JS{\mathrm{JS}}
\def\eps{\epsilon}

\newtheorem{definition}{Definition}
\newtheorem{theorem}{Theorem}
\newtheorem{corollary}{Corollary}
\newtheorem{lemma}{Lemma}

\graphicspath{{./Fig/}}
\title{On $w$-mixtures:\\ Finite convex combinations of prescribed component distributions\footnote{A preliminary version of these results appeared in~\cite{wmixturegeometry-2018} (IEEE ICASSP 2018).}}
\date{}
 
\author{Frank Nielsen\footnote{Frank Nielsen is with Sony Computer Science Laboratories Inc (Japan). E:mail: {\tt Frank.Nielsen@acm.org}}
\and
Richard Nock\footnote{Richard Nock is with Data61, the Australian National University, and the University of Sydney, Australia. {\tt Richard.Nock@data61.csiro.au}}
}

\begin{document}
\maketitle
 
\begin{abstract}
We consider the space of $w$-mixtures which is defined as the set of finite statistical mixtures sharing the same prescribed component distributions closed under convex combinations.
The information geometry induced by the Bregman generator set to the Shannon negentropy on this space yields a dually flat space called
the mixture family manifold. 
We show how the Kullback-Leibler (KL) divergence can be recovered from the corresponding Bregman divergence for the negentropy generator:
That is, the KL divergence between any two $w$-mixtures amounts to a Bregman Divergence (BD) induced by the Shannon negentropy generator. 
Thus the KL divergence between two Gaussian Mixture Models (GMMs) sharing the same Gaussian components is equivalent to a Bregman divergence.
This KL-BD equivalence on a mixture family manifold implies that we can perform optimal KL-averaging aggregation of $w$-mixtures without information loss. 
More generally, we prove that the statistical skew Jensen-Shannon divergence between $w$-mixtures is equivalent to a skew Jensen divergence between their corresponding parameters. 
Finally, we state several properties, divergence identities, and divergence inequalities relating to $w$-mixtures.
\end{abstract}
 
\noindent Keywords:
Mixture family manifold; Shannon negentropy; Kullback-Leibler divergence; Bregman divergence;  Legendre-Fenchel divergence; \mbox{$f$-divergences}; total variation; Jensen-Shannon divergence; Jensen divergence; information geometry; distributed statistical estimation.

\section{Introduction}

In the field of statistics, finite mixtures~\cite{statmixture-2004}  are {\em semi-parametric models} defined according to weighted component distributions.
When the component distributions belong to a same parametric family of distributions, the mixture is said {\em homogeneous}, otherwise it is said
 {\em heterogeneous}. 
For example, a {\em Mixture of Gaussians} (MoGs, i.e., a mixture of normal distributions)  is a homogeneous mixture
commonly called a {\em Gaussian Mixture Model}~\cite{Goldberger-2005} (GMM).
A mixture of a Laplace distribution with a Gaussian distribution is an example of an heterogeneous mixture.

In this work, we shall consider $w$-mixtures which are convex weighted combinations of {\em fixed} component distributions.

Let $M_+^1(\Omega)$ denote the space of probability measures defined on a $\sigma$-algebra $\Omega$ of a sample space $\calX$.
Consider a {\em positive base measure} $\mu\in M_+^1(\Omega)$ (e.g., the Lebesgue measure  or the counting measure), 
and let $P_0, \ldots, P_{k-1}$ be $k$ {\em prescribed} probability distributions, all dominated by the base measure  $\mu$ ($P_i\ll \mu$), 
with respective densities $p_0=\frac{\mathrm{d}P_0}{\mathrm{d}\mu},\ldots, p_{k-1}=\frac{\mathrm{d}P_{k-1}}{\mathrm{d}\mu}$ (i.e., $p_i$ is the Radon-Nikodym derivative of $P_i$ with respect to $\mu$).

We define $w$-mixtures by their densities as a weighted average of component densities as follows:

\begin{definition}[Statistical $w$-mixtures]
The density $m(x;w)\in M_+^1(\Omega)$ of a {\em $w$-mixture} is defined by: 
$$
m(x;w) \eqdef \sum_{i=0}^{k-1} w_i p_i(x),
$$
with $w\eqdef (w_0,\ldots, w_{k-1})\in \Delta_{k-1}^\circ$, where $\Delta_{k-1}^\circ=\{ w\in\bbR^{k}\ :\ \sum_{i=0}^{k-1} w_i=1 \}$ is the $(k-1)$-dimensional open probability simplex sitting in $\mathbb{R}^k$   with $\sum_{i=0}^{k-1} w_i=1$
(called $(k-1)$-simplex or standard simplex).
\end{definition}
 
That is, $w$-mixtures are {\em strictly} convex weighted combinations of $k$ {\em fixed} component distributions:
They form {\em special subfamilies} of finite statistical mixtures~\cite{statmixture-2004} that are
closed by convex combinations: 
Indeed, one can check that the mixture of $w$-mixtures is a $w$-mixture: 
$$
(1-\alpha)m(x;w)+\alpha m(x;w') = m\left(x;(1-\alpha)w+\alpha w'\right),\quad \forall\alpha\in [0,1].
$$

To motivate the use of $w$-mixtures in applications, let us report some prior works which built $w$-mixtures:

\begin{itemize}
\item Given $n$ datasets $\calO_1,\ldots,\calO_n$, a set of $n$ $w$-mixtures  $m_1=(x;w_1),\ldots, m_n=m(x,w_n)$  (called comixs in~\cite{comix-2016})
can be {\em jointly learned}   by generalizing the
 {\em Expectation Maximization} algorithm~\cite{statmixture-2004} (EM) or the {\em Classification EM} (CEM) algorithm.
In particular, one can learn $w$-Gausian Mixture Models~\cite{comix-2016} ($w$-GMMs for short) where the prescribed mixture components are {\em fixed} Gaussian distributions.

\item Another way to obtain $w$-mixtures is to consider {\em Kernel Density Estimators} (KDEs) at {\em prescribed} locations~\cite{schwander2013,Kernel-1994} for several datasets $\calO_1,\ldots,\calO_n$.
\end{itemize}

Figure~\ref{fig:wgmm} displays two $w$-GMMs with three ($k=3$) fixed components.

\begin{figure}%
\centering
\includegraphics[width=0.8\columnwidth]{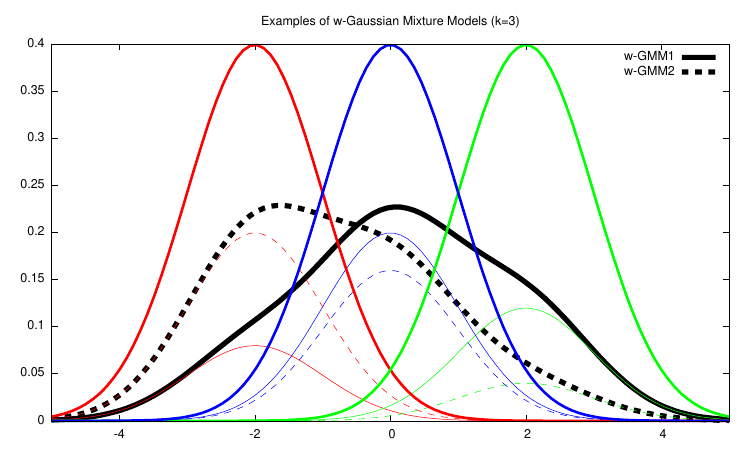}%

\caption{Two $w$-Gaussian Mixture Models ($w$-GMMs) $m(x;w)$ (plain) and $m(x;w')$ (dashed) with fixed components $p_0(x)\sim \calN(-2,1)$ (red), $p_1(x)\sim \calN(2,1)$ (green), and 
$p_2(x)\sim \calN(0,1)$ (blue), and weight vectors $w=(0.2,0.3,0.5)$ and $w'=(0.5,0.1,0.4)$.
 }%
\label{fig:wgmm}%
\end{figure}
 
In this paper, we study the information-geometric structure of the manifold $\calM$ of $w$-mixtures~\cite{Vos-1991,IGDiv-2010}:
$$
\calM\eqdef\{m(x;w)\ ,\ w\in\Delta_{k-1}^\circ\},
$$
and state properties related to several information-theoretic divergences (e.g., total variation metric distance, Kullback-Leibler divergence, Jensen-Shannon divergence, etc.).

The paper is organized as follows:

First, we  concisely describe in \S\ref{sec:igfdiv} the generic construction of the information geometry induced by an arbitrary smooth (parameter) divergence~\cite{IGDiv-2010,EIG-2018}, and recall 
the basics of the class of statistical invariant $f$-divergences~\cite{Vos-1991,IGDiv-2010,IG-2016}.
Next, we describe in~\S\ref{sec:igdf} 
 the dually flat geometry of the space of $w$-mixtures induced either by the Kullback-Leibler (KL) divergence or equivalently by its Bregman  generator set to  the Shannon negentropy (called Shannon information).
This information-geometric construction implies that the KL divergence between any two $w$-mixtures is mathematically equivalent to a Bregman divergence induced by the negative Shannon entropy generator which is however often not available in closed-form\footnote{The Shannon negentropy is available in closed-form for the family of categorical distributions (mixture of Dirac distributions), and more generally, when the component distributions have pairwise disjoint supports.} for $w$-mixtures. 
Yet this structural observation allows us to prove in~\S\ref{sec:igdf} that the {\em KL-averaging aggregation} of $w$-mixtures  can be performed optimally without information loss which is useful for distributed inference of $w$-mixtures~\cite{DistributedEstimation-EF-2014}. 
In~\S\ref{sec:bdjd}, we extend the KL-BD equivalence and show that the skew Jensen-Shannon divergences~\cite{JS-2019,JS-2020} of $w$-mixtures amount 
to skew Jensen $\alpha$-divergences~\cite{BR-2011} on their corresponding weight parameters.
Finally, we consider several divergence inequalities~in \S\ref{sec:closure} between $w$-mixtures.
In the appendix, we recall in \S\ref{sec:commondist} the principal divergences.
Then we show how to stochastically estimate by the Monte Carlo method $f$-divergences which are guaranteed to be non-negative in~\S\ref{sec:extendfdiv} (highlighting a connection between $f$-divergences and Bregman divergences).
Finally, in Appendix~\S\ref{sec:BDStatDiv}, we describe a method to build statistical distances from Bregman divergences induced by integral-based Bregman generators.

\section{Information geometry and $f$-divergences\label{sec:igfdiv}}

\subsection{Information geometry induced by a divergence}

A {\em divergence} $D(p:q)$ is a smooth measure of dissimilarity that satisfies $D(p:q)\geq 0$ with equality iff. $p=q$ (reflexivity property).
The ``:'' notation emphasizes that the divergence may potentially be {\em asymmetric}: $D(p:q)\not= D(q:p)$.
Divergences are also called {\em contrast functions}~\cite{Eguchi-1992}.
(Divergences in information geometry should not be confused with the divergence operator on vector fields.)

Because of the potential asymmetry of a divergence, we define a {\em dual divergence} $D^*(p:q)\eqdef D(q:p)$ via {\em reference duality}~\cite{refduality-2015}, and a {\em symmetrized divergence} $S(p;q) \eqdef \frac{1}{2}(D(p:q)+D^*(p:q))=S(q;p)$ (that may not satisfy the triangle inequality of metric distances).
We emphasize that symmetrized divergences are symmetric divergences using the ``;'' notation instead of the ``:'' notation to separate their arguments.

In information geometry~\cite{Calin-2014,IG-2016,EIG-2018}, we equip a manifold $\calM$ with a {\em metric tensor} and a {\em pair of dual torsion-free affine connections}. 
The  structure $(\calM,D)$ can be induced by any smooth $C^\infty$ divergence $D(\cdot:\cdot)$~\cite{Eguchi-1992,IGDiv-2010,Calin-2014,IG-2016} as follows: 

\begin{enumerate}
\item  The {\em metric tensor}
$g(p)$
	  provides an {\em inner product} between vectors at each tangent plane $T_p$: $\inner{v}{v'}_p= \sum_{i,j} g_{ij} v_i v_j'$.
	These local inner products vary {\em smoothly} on $\calM$, and are used for measuring 
	\begin{itemize}
	\item {\em angles} $\arccos \inner{v}{v'}_p$ between vectors $v$ and $v'$ (with $v\perp v'$ iff $\inner{v}{v'}_p=0$), and
	\item vector {\em lengths} $\|v\|_p:=\sqrt{\inner{v}{v}}_p$ on any tangent plane $T_p$ for $p\in\calM$.
	\end{itemize}

\item  A {\em pair of dual torsion-free affine connections} $\nabla$ and $\nabla^*$ for defining how  vectors are ``parallel'' transported  from any source tangent plane $T_p$ to any  target tangent plane $T_{q}$.
	In differential geometry, an affine {\em connection} $\nabla$ allows one to define {\em $\nabla$-geodesics} as auto-parallel curves $\gamma$: $\nabla_{\dot \gamma} \dot \gamma=0$.
	The symbol $\nabla$ is also used to define a differential operator, called the {\em covariant derivative} acting on vector fields (and more generally any type of tensors).
	An affine connection $\nabla$ is defined by its Christoffel symbols $\Gamma_{ijk}$ which are smooth functions expressed in a local coordinate chart.
	\end{enumerate}
	
	Given a divergence $D(\cdot:\cdot)$, we can induce~\cite{Eguchi-1992,IGDiv-2010} the metric tensor by 
	\begin{equation}
	g_{ij}(p)= \left. \frac{\partial^2}{\partial x_i\partial x_j} D(x:y)\right\vert_{y=x},
	\end{equation}
	and the corresponding affine connection $\nabla$ with Christoffel symbols 
	\begin{equation}
	\Gamma_{ijk}(x)= \left. -\frac{\partial^3}{\partial x_i\partial x_j\partial y_k} D(x:y)\right\vert_{y=x}.
	\end{equation}
	
	The dual connection $\nabla^*$ is induced by the dual divergence $D^*(\cdot:\cdot)$ (with dual metric tensor $g^*=g$),
	and
the pair of connections $(\nabla,\nabla^*)$ is said dually coupled to the metric tensor $g$ since for any triple of vector fields $X, Y$ and $Z$, we have~\cite{Calin-2014,IG-2016,EIG-2018} the following property:
\begin{equation}
	\nabla_X\inner{Y}{Z} = \inner{\nabla_X Y}{Z} + \inner{Y}{\nabla^*_X Z}.
\end{equation}
	
The pair $(\nabla,\nabla^*)$ is said {\em conjugate} because $\calM$ being flat wrt. to $\nabla$ implies $\calM$ being flat wrt. to $\nabla^*$, and vice-versa~\cite{Eguchi-1992,IGDiv-2010,EIG-2018}.
Furthermore, the mean connection $\bar{\nabla}=\frac{\nabla+\nabla^*}{2}$ corresponds to the Riemannian torsion-free Levi-Civita metric connection  $\nabla^g$ which is induced by the metric tensor $g$.
One can define a third-order totally symmetric tensor $T$~\cite{Eguchi-1992,IGDiv-2010,EIG-2018} (called the skewness tensor) for vector fields $X, Y$ and $Z$ by:
$$
T(X,Y,Z) \eqdef g(\nabla_X Y-\nabla^*_X Y,Z).
$$ 
The structure $(\calM,g,T)$ is called a {\em statistical manifold}~\cite{lauritzen1987statistical}, and characterizes the dualistic structure of information geometry~\cite{Calin-2014,IG-2016,EIG-2018}.

\begin{table}
$$
\begin{array}{lll}
\text{Name} & \text{$f$-divergence $I_f(p:q)$} & \text {Generator $f(u)$}\\
\hline\hline
\text{Total variation (metric)} & \frac{1}{2}\int |p(x)-q(x)| \dnu(x) & \frac{1}{2} |u-1| \\
\text{Squared Hellinger} & \int (\sqrt{p(x)}-\sqrt{q(x)})^2 \dnu(x) & (\sqrt{u}-1)^2\\
\text{Pearson $\chi^2_P$}  &  \int \frac{(q(x)-p(x))^2}{p(x)} \dnu(x) & (u-1)^2\\
\text{Neyman $\chi^2_N$}  &  \int \frac{(p(x)-q(x))^2}{q(x)} \dnu(x) & \frac{(1-u)^2}{u}\\
\text{Kullback-Leibler} & \int p(x)\log \frac{p(x)}{q(x)} \dnu(x) & -\log u\\
\text{reverse Kullback-Leibler} & \int q(x)\log \frac{q(x)}{p(x)} \dnu(x) & u\log u \\
\text{Squared triangular} & \int \frac{(p(x)-q(x))^2}{p(x)+q(x)} \dnu(x) & \frac{(u-1)^2}{2(1+u)}\\
\text{Squared perimeter} & \int \sqrt{p^2(x)+q^2(x)}  \dnu(x) -\sqrt{2} & \sqrt{1+u^2}-\frac{1+u}{\sqrt{2}}\\
\text{$\alpha$-divergence} &  \frac{4}{1-\alpha^2} (1-\int p^{\frac{1-\alpha}{2}}(x) q^{1+\alpha}(x) \dnu(x))  & \frac{4}{1-\alpha^2}(1-u^{\frac{1+\alpha}{2}})\\ 
\text{Jensen-Shannon} & \frac{1}{2}\int (p(x)\log \frac{2p(x)}{p(x)+q(x)} +  q(x)\log \frac{2q(x)}{p(x)+q(x)})\dnu(x) &  -(u+1)\log \frac{1+u}{2} + u\log u\\ \hline
\end{array}
$$
 
\caption{Common $f$-divergences $I_f(p:q)$ with their corresponding convex generators $f(u)$.\label{tab:fdiv}}
\end{table}
  
\subsection{$f$-Divergences: Definition, properties, and stochastic estimation}
The class of statistical {\em $f$-divergences}~\cite{Morimoto-1963,Csiszar-1963,AliSilvey-1966} between two distributions $p,q\ll\mu$ defined on support $\calX$ is defined by:
\begin{equation}
I_f(p:q) \eqdef \int_\calX p(x)f\left(\frac{q(x)}{p(x)}\right) \dmu(x) \geq f(1),
\end{equation} 
where $f(u)$  is a convex function satisfying $f(1)=0$ and strictly convex at $1$. 
We use the following conventions:
$$
0f\left(\frac{0}{0}\right)=0,\quad f(0)=\lim_{u\rightarrow 0^+} f(u),\quad \forall a>0, 0f\left(\frac{a}{0}\right)=\lim_{u\rightarrow 0^+} uf\left(\frac{a}{u}\right)=a\lim_{u\rightarrow \infty} \frac{f(u)}{u}.
$$

Two generators $f(u)$ and $g(u)$ induce the same $f$-divergence, $I_f(p:q)=I_g(p:q)$, iff there exists $\lambda\in\bbR$ such that $g(u)=f(u)+\lambda(u-1)$. Thus we may enforce wlog. that $f'(1)=0$. We may further fix the scale of the $f$-divergence by setting $f''(1)=1$.

For discrete distributions with Probability Mass Functions (PMFs)
 $p=(p_0,\ldots, p_{d-1})$ and $q=(q_0,\ldots, q_{d-1})$, it comes that $I_f(p:q)=\sum_{i=0}^{d-1} p_i f(\frac{q_i}{p_i})$ (base measure $\mu$ is the counting measure).
Common $f$-divergences~\cite{fdivapprox-2014} include the Kullback-Leibler (KL) divergence ($f(u)=-\log u$), the $\chi^2$-divergence, the Hellinger divergence, the $\alpha$-divergences,  the total variation $\TV(p,q)\eqdef\frac{1}{2}\int_\calX |p(x)-q(x)|\dmu(x)$ ($f(u)=\frac{1}{2}|1-u|$, etc. See Table~\ref{tab:fdiv} for a summary of the main $f$-divergences.
The only $f$-divergence which is a metric distance~\cite{TVuniquefmetric-2007} satisfying the triangle inequality is the total variation distance.

The dual $f$-divergence $I_f^*(p:q)\eqdef I_f(q:p)$ is obtained by taking the {\em dual generator} $f^\diamond(u) \eqdef uf\left(\frac{1}{u}\right)$: $I_{f^\diamond}(p:q)=I_f(q:p)=I_f^*(p:q)$. Thus $f$-divergences can always be symmetrized by taking the generator $s(u)=\frac{1}{2}(f(u)+f^\diamond(u))$.
Examples of symmetric $f$-divergences are the Jeffreys divergence~\cite{IG-2016} $J(p;q) \eqdef \KL(p:q) +  \KL(q:p)$  and the Jensen-Shannon divergence~\cite{js-1991,Nielsen-2010} $\JS(p:q)\eqdef \frac{1}{2} (K(p:q) + K(q:p))$ with $K(p:q) \eqdef \KL(p:\frac{p+q}{2})=\int p(x)\log\frac{2p(x)}{p(x)+q(x)}\dmu(x)$.
The $f$-divergences are upper bounded by~\cite{Vajda-1987}: 
$$
I_f(p:q)\leq \lim_{\eps\rightarrow 0} f(\eps)+ f^\diamond(\eps).
$$
Thus when $f(0)+f^\diamond(0)<\infty$, the symmetrized $f$-divergences are bounded (e.g., the Jensen-Shannon divergence~\cite{JS-2019,JS-2020} re always bounded by $\log 2$ or the total variation distance is bounded by $1$).
Depending on the generator $f$, the $f$-divergence may be unbounded or even be infinite when the integral diverges: $I_f(p:q) \eqdef +\infty$ (e.g.,
 the KL divergence between a standard Cauchy distribution and a standard normal distributions).
The $f$-divergences can be extended~\cite{DivEstimator-2006} to positive measure $p$ and measure $q$ (potentially negative) by taking the
extended  generator 
$\bar{f}(u)=f(u)-f'(1)(u-1)$ 
for $f$ a continuously differentiable function at $u=1$ (thus $\bar{f}(1)=\bar{f}'(1)=0$).

The $f$-divergences between statistical mixtures~\cite{KLGMM-LSE-2016,nielsen-2012} is not available in closed form although it can be easily upper bounded by using the {\em joint convexity property} of $f$-divergences~\cite{Vajda-1987}: 
$$
I_f(m:m')\leq \sum_{i,j} w_iw_j' I_f(p_i:p_j')
$$ 
for two mixture models $m(x)=\sum_i w_ip_i(x)$ and $m'(x)=\sum_j w_j'p_j'(x)$.

In practice, to bypass the intractability of $f$-divergences, one {\em estimates} the $f$-divergence using Monte Carlo (MC) stochastic integration (see~\cite{DL-2016}, Chapter 17):
Let $s$ iid. samples $x_1,\ldots, x_s \sim p(x)$, and define the estimator 
$$
\hat I_f^s(p:q) \eqdef \frac{1}{s} \sum_{i=1}^s f\left( \frac{q(x_i)}{p(x_i)} \right).
$$

It follows from the Law of Large Numbers (LLN) that  $\lim_{s\rightarrow} \hat I_f^s(p:q)=I_f(p:q)$ provided that  the variance $V_p\left[f\left( \frac{q(x)}{p(x)} \right)\right]$ is bounded. The MC estimator is {\em consistent} when $I_f(p:q)<\infty$ but  but the MC approximation does not hold when $I_f$ diverges.
Furthermore, using the Central Limit Theorem (CLT), the MC estimator is shown to be {\em normally} distributed: 
$$
\hat I_f^s(p:q)\sim \mathcal{N}\left(I_f(p:q),\frac{1}{s}V_p\left[f\left( \frac{q(x)}{p(x)} \right)\right]\right).
$$

In practice, the variance is approximated by the sample variance, and Confidence Intervals (CIs) can be reported.
 
Note that the MC estimator $\hat I_f^s(p:q)$ may not guarantee that $\hat I_f^s(p:q)\geq 0$:
It can potentially yield a negative number. This is because $f\left(\frac{q(x)}{p(x)}\right)$ may be negative.

For the KL MC estimator, we use the following MC estimation:
$$
\widehat{\KL}^s(p:q) \eqdef \frac{1}{s} \sum_{i=1}^s \left(  \log \frac{p(x_i)}{q(x_i)}  + \frac{q(x_i)}{p(x_i)} -  1 \right),
$$
which is the MC estimator for the {\em extended KL divergence}~\cite{Bregman-2005} (extended to arbitrary positive measures).
In particular, the extended KL divergence amounts to the KL divergence for normalized densities.
This (extended) KL estimator is {\em guaranteed} to be non-negative since we can rewrite it as:
$$
\widehat{\KL}^s(p:q) \eqdef \frac{1}{s} \sum_{i=1}^s \IS(q(x_i) : p(x_i)),
$$
where
$$
\IS(p:q)=\frac{p}{q}+\log\frac{q}{p}-1\geq 0,
$$
is the univariate {\em Itakura-Saito divergence}~\cite{Bregman-2005} (hence nonnegative).

In general, MC estimation of $f$-divergences may violate the reflexivity property of divergences (i.e., $D(p:q)=0 \Leftrightarrow p=q$).
Indeed, consider the generator $f_\lambda(u)=f(u)+\lambda(u-1)$ for $\lambda\in\bbR$.
We have $I_f(p:q)=I_{f_\lambda}(p:q)$.
The MC estimation of $f_\lambda$-divergence yields

\begin{eqnarray*}
\hat I_{f_\lambda}^s(p:q) &\eqdef & \frac{1}{s} \sum_{i=1}^s f_\lambda\left( \frac{q(x_i)}{p(x_i)} \right),\\
&=&  \frac{1}{s} \sum_{i=1}^s \left( f\left( \frac{q(x_i)}{p(x_i)}  \right) + \lambda \left(\frac{q(x_i)}{p(x_i)}-1\right) \right).
\end{eqnarray*}

Choosing $\lambda_0=\frac{ \sum_{i=1}^s  f\left( \frac{q(x_i)}{p(x_i)}  \right) }{ \sum_{i=1}^s  \left(\frac{q(x_i)}{p(x_i)}-1\right)}$ yields
$\hat I_{f_{\lambda_0}}^s(p:q)=0$ although $I_{f_{\lambda_0}}(p:q)>0$ for $q\not =p$.
Note that $\widehat{\KL}^s(p:q)=0$ iff $p=q$ when $s$ is greater than the number of degrees of freedom used for describing the distributions. 
Notice that in practice, one has to take care of numerical precision errors when implementing MC stochastic estimator.

Appendix~\ref{sec:extendfdiv} further report details on {\em extended $f$-divergences} with guaranteed non-negative Monte Carlo estimators by highlighting a connection between $f$-divergences and Bregman divergences.
We also refer the reader to~\cite{fdivestimate-2014,fdivkNN-2017} for other techniques for efficiently estimating $f$-divergences.

\subsection{Information monotonicity and invariance}
To get lower bounds on the $f$-divergence $I_f$, we  use the {\em information monotonicity} property of $f$-divergences~\cite{Calin-2014,IG-2016,EIG-2018}. 
Let $\calA=\uplus_{i=1}^h \calA_i$ be a {\em partition} of the support $\calX$ into $h$ pieces.
Let $\tp=(\tp_0, \ldots, \tp_{h-1})$ and $\tq=(\tq_0, \ldots, \tq_{h-1})$ denote the discrete distributions obtained by {\em coarse-graining} $p$ and $q$, with $\tp_i=\int_{\calA_i} p(x)\dmu(x)$ and $\tq_i=\int_{\calA_i} q(x)\dmu(x)$. This coarse-graining process is called lumping in~\cite{Csiszar-2004} and is a particular case of a Markov kernel. 
Coarse-graining can be interpreted as converting distributions into histograms with $h$ bins.
Then the information monotonicity~\cite{IG-2016} of  divergences (related to the data processing inequality~\cite{DPI-1997}, DPI) ensures that 

$$
0\leq D(\tp:\tq) \leq D(p:q).
$$
Clearly, when $h=1$ and there is a unique bin, we have $0= D(1:1) \leq D(p:q)$.
In particular, this lower bound applies when $p=m(x,w)$ and $q=m(x,w')$ are two $w$-mixtures.

A divergence is said {\em separable} iff $D(p:q)=\sum_i D_1(p_i:q_i)$ where $D_1$ is a {\em scalar divergence}.
$f$-Divergences are the only separable divergences (also called decomposable divergences), with $I_f^1(p:q)=pf(q/p)$,  that enjoy information monotonicity~\cite{IG-2016,Liang-2016} (except for the special case of binary alphabets~\cite{Jiao-2014}).
Since $I_f(\tp:\tq) =\sum_{i=0}^{h-1} \tp_i f\left(\frac{\tq_i}{\tp_i}\right)$ is computable in $O(h)$ time, it yields a lower bound on  $I_f(p:q)$ provided that we are    able to compute in closed-form $\tp$ and $\tq$, say, using Cumulative Distribution Functions (CDFs). Usually, CDF formula are available for univariate distributions (say, Gaussian) but this is often not tractable for multivariate distributions although efficient numerical schemes are available~\cite{Genz-2009}.
 Notice that $f$-divergences can be lower bounded using total variation distances by Pinsker-type inequalities~\cite{Pinsker-2009}.
Finally, let us state that $f$-divergences  are invariant~\cite{finvar-2010}  under differentiable and invertible transformations of the sample space (say, $h:\calX\mapsto\calY$):
$I_f(p:q)=I_f(p':q')$ with $p'(x)=p(h(x))=p(x)|J(x)|^{-1}$ and $q'(x)=q(h(x))=q(x)|J(x)|^{-1}$, where $|J(x)|$ denotes the determinant of the Jacobian matrix of $h$.

\section{Geometry induced by the Kullback-Leibler divergence}\label{sec:igdf}

In this section, we recall the fact that the space of $w$-mixtures forms a mixture family in information geometry, and as such, we can model the space of $w$-mixtures
as a dually flat space equipped with a pair of convex potential functions inducing dual Bregman divergences~\cite{EIG-2018,LegendreIG-2010}.
We provide a full description of this fact mentioned  in~\cite{IGHierarchy-2001,DeformedIG-2012,Calin-2014,IG-2016}, and perform sanity checks during the construction. See also~\cite{MCIG-2018,MCIG-2019}.

\subsection{Dually flat manifold of $w$-mixtures}
When the $k$ prescribed component distributions $p_0(x),\ldots, p_{k-1}(x)$ are {\em linearly independent}, 
the space 
$$
\calM=\left\{m(x;w)\ ,\ w\in\Delta_{k-1}^\circ\right\},
$$ 
of $w$-mixtures forms a {\em mixture family} in information geometry~\cite{IG-2016}:
$$
\calM=\left\{ m(x;\eta)= \sum_{i=1}^{k-1} \eta_i p_i(x) + \left(1-\sum_{i=1}^{k-1} \eta_i\right) p_0(x), \eta\in\bbR^{k-1}_{++}  \right\},
$$
with $\eta_i=w_i$ for $i\in [k-1]=\{1,\ldots,k-1\}$ and $w_0=1-\sum_{i=1}^{k-1} \eta_i = 1-\sum_{i=1}^{k-1} w_i$.
(We adopted the ``primal'' parameter to be $\eta$ following the textbook~\cite{IG-2016}; Appendix~\ref{sec:BDStatDiv} explains the dually flat construction with primal parameter $\theta$.)
We have $\sum_{i=1}^{k-1} \eta_i<1$. Denote by $H^\circ\eqdef \{ \eta\in\bbR^{k-1}_{++} \ :\ \sum_{i=1}^{k-1} \eta_i<1\}$. 
Let $D=k-1$ denote the {\em order} of the mixture family, that is its number of degrees of freedom. 
We have $m(x;w)=m(x;\eta)$, where vector $w$ is $k$-dimensional while vector $\eta$ is $(k-1)$-dimensional.
Manifold $\calM$ is an affine subspace of the space of density wrt to $\mu$.

Let $f_i(x)=p_i(x)-p_0(x)$ for $i\in [D]$ (with $\int_\calX f_i(x)\dmu(x)=0$), and $c(x)=p_0(x)$ (with $\int_\calX p_0(x)\dmu(x)=1$).
Then $\calM$ can be written in the {\em canonical form} of a mixture family in  information geometry~\cite{IG-2016}:
$$
\calM=\left\{ m(x;\eta)= \sum_{i=1}^{k-1} \eta_i f_i(x) + c(x),\quad \eta\in H^\circ  \right\},
$$
where the $f_i(x)$'s and $c(x)$ are linearly independent.
By convention, we shall denote by $\eta_0 = 1-\sum_{i=1}^{D} \eta_i$ the weight of $p_0$. 
Beware that $\eta_0$ is {\em not} a vector component of $\eta=(\eta_1,\ldots,\eta_D)\in H^\circ$, the $D=(k-1)$-dimensional open probability simplex sitting in $\mathbb{R}^d$.
We can convert from weight coordinate $w$ to $\eta$-coordinate as follows:
$$
w=\left(
\begin{array}{c}
w_0=1-\sum_{i=1}^{k-1} w_i\\
w_1\\
\vdots\\
w_{k-1}
\end{array}
\right)\in\Delta_{k-1}^\circ\subset\bbR^k
\Leftrightarrow
\eta=\left(
\begin{array}{l}
\eta_1=w_1\\
\vdots\\
\eta_{k-1}=w_{k-1}
\end{array}
\right)\in H^\circ\subset\bbR^{k-1}_{++}\subset\bbR^D.
$$

We consider $\calM$ as a {\em smooth manifold} of $\eta$-mixtures.
The Shannon differential entropy~\cite{Cover-2012,ShannonDiscontinuity-2009} 
\begin{equation}
h(m) \eqdef -\int_\calX m(x)\log m(x)\dmu(x),
\end{equation}
of a mixture $m(x)$ is usually not available in closed-form~\cite{KLGMM-LSE-2016} because of the log-sum term.
Hower, both lower and upper bounds on the entropy of mixtures are given in~\cite{KLGMM-LSE-2016,UBEntropyMixture-2017}:

\begin{eqnarray*}
h(m) &\geq&  \sum_i w_i h(p_i)  - \sum_{i} w_i \log\left( \sum_{j} w_j e^{-\Bhat(p_i:p_j)} \right),\\
h(m) &\leq&  \sum_i w_i h(p_i)  - \sum_{i} w_i \log\left( \sum_{j} w_j e^{-\KL(p_i:p_j)} \right),
 \end{eqnarray*}
where $\Bhat$ denotes the {\em Bhattacharrya divergence}~\cite{Bhat-1946,BR-2011} defined by
\begin{equation*}\label{eq:bhat}
\Bhat(p_i:p_j) \eqdef -\log \int_\calX \sqrt{p_i(x)p_j(x)}\dmu(x).
\end{equation*}

\subsection{Negative entropy as a potential convex function: Bregman generator}
For $\eta$-mixtures, the {\em parametric function} $E(\eta)=F^*(\eta)=-h(m(x;\eta))$ (the notation $F^*$ will be explained shortly thereafter),
is strictly convex and differentiable~\cite{Calin-2014}.
 For example, when $D=1$, we have $(F^*(\eta))''=\int_\calX  \frac{(p_1(x)-p_0(x))^2}{m(x;\eta)} \dmu(x)>0$, and therefore $E$ is strictly convex and differentiable.
Figure~\ref{fig:graph} displays the graph of the negative entropy $F^*(\eta)$ for two $\eta$-mixtures of order $D=1$.

\begin{figure}%
\centering
\input{Fig/Gconvex12.latex}%
\caption{Graph plots of the negative entropy $F^*(\eta)=-h(m(x;\eta))$ for two $\eta$-mixtures of order $D=1$ (with univariate Gaussian components).
Here, the function $F^*(\eta)$ is estimated using Monte-Carlo integration with $10^6$ samples.}%
\label{fig:graph}%
\end{figure}

In general, the Hessian $\nabla^2 F^*(\eta)$ matrix has coefficients
\begin{eqnarray}
\nabla^2 F^*(\eta)_{i,j} &=& \int_\calX  \frac{(p_i(x)-p_0(x))  (p_j(x)-p_0(x))}{m(x;\eta)} \dmu(x),\\
&=& E_{x\sim m(x;\eta)}\left[\frac{(p_i(x)-p_0(x))  (p_j(x)-p_0(x))}{m^2(x;\eta)}\right].
\end{eqnarray}
It is a positive definite matrix: $\nabla^2 F^*(\eta)\succ 0$.
Thus we can  form a {\em dually flat manifold}~\cite{Calin-2014,IG-2016} 
where the Kullback-Leibler divergence between two mixtures $m(x;\eta_1)$
and $m(x;\eta_2)$ amounts to calculate a {\em Bregman divergence}~\cite{Bregman-2005} $B_{F^*}(\eta_1:\eta_2)$ for the {\em negative Shannon information generator}~\cite{IG-2016}:
\begin{eqnarray}\label{eq:FShannon}
F^*(\eta)&=& -h(m(x;\eta))\\
&=& \int_\calX m(x;\eta)\log m(x;\eta)  \dmu(x).
\end{eqnarray}
Since Shannon entropy is strictly concave, the negative Shannon entropy called {\em Shannon information} is strictly convex (and a dually flat manifold can be built from any convex function). Let $m_1(x)\eqdef m(x;\eta_1)$ and $m_2(x) \eqdef m(x;\eta_2)$.
We have

\begin{eqnarray}
\KL(m_1:m_2)&\eqdef&  \int_\calX m(x;\eta_1) \log\frac{m(x;\eta_1)}{m(x;\eta_2)} \dmu(x),\\
&=& F^*(\eta_1) - F^*(\eta_2) - \inner{\eta_1-\eta_2}{\nabla F^*(\eta_2)},\label{eq:BDFstar}\\
&=& B_{F^*}(\eta_1:\eta_2),
\end{eqnarray}
where $\inner{x}{y}=x^\top y$ denotes the scalar product of $\mathbb{R}^D$.
Although the Shannon information of a $w$-mixture is a convex function of $\eta$, it is  usually not available in closed-form~\cite{KLnotanalytic-2004,HGMM-2016}. A very particular case is the mixture family of two Cauchy distributions for which the Shannon entropy is available in closed-form~\cite{nielsen2021dually}.
The $\eta$ parameter is traditionally called the {\em ``expectation'' parameter} in information geometry (although this stems from a property of the exponential family manifolds~\cite{IG-2016}). 
We shall write for short $h(\eta)=h(m(x;\eta))$.

Since 
\begin{equation}
F^*(\eta_1)= \int_\calX m(x;\eta_1) \log m(x;\eta_1) \dmu(x)=-h(m(x;\eta_1)),
\end{equation}
 it follows from Eq.~\ref{eq:BDFstar}
that 
\begin{equation}
\int_\calX m(x;\eta_1) \log\frac{1}{m(x;\eta_2)} \dmu(x)=-\int_\calX m(x;\eta_1) \log {m(x;\eta_2)} \dmu(x)
\end{equation}
 is the cross-entropy~\cite{crossEF-2010} $h^\times(m(x;\eta_1):m(x;\eta_2))$ (with $h^\times(\eta:\eta)=h(m(x;\eta))$) and we have

\begin{eqnarray}
h^\times(m(x;\eta_1):m(x;\eta_2)) &=& -\int_\calX m(x;\eta_1) \log {m(x;\eta_2)} \dmu(x),\\
&=&   - F^*(\eta_2) - \inner{\eta_1-\eta_2}{\nabla F^*(\eta_2)}.
\end{eqnarray}

\begin{figure}%
\centering

\input{Fig/Fconvex12.latex}%

\caption{Graph plots of the cross-entropy $F(\theta)=h^\times(p_0(x):m(x;\eta))$ for two $\eta$-mixtures of order $1$ (with univariate Gaussian components).
The function $F^(\theta)$ is estimated using Monte-Carlo integration with $10^6$ samples.}%
\label{fig:graphF}%
\end{figure}

\subsection{Cross-entropy as the dual potential function: Dual Bregman generator}

The dual parameters $\theta=(\theta^1,\ldots,\theta^D)$, called the {\em natural parameters}, are defined by

\begin{equation}\label{eq:theta}
\theta^i(\eta) = (\nabla_\eta F^*(\eta))_i =\int_\calX \left(p_i(x)-p_0(x)\right)\log m(x;\eta)\dmu(x), 
\end{equation}
since $(\nabla_\eta m(x;\eta))_i=p_i(x)-p_0(x)$ and swapping $\nabla\int=\int\nabla$ (under regularity condition of Leibniz integral rule).
Figure~\ref{fig:graphF} displays the graph plot of $F(\theta)$ for uni-order $w$-Gaussian mixture models.
We can rewrite the natural parameter coordinates as

$$
\theta^i(\eta) = h^\times(p_0(x):m(x;\eta)) -  h^\times(p_i(x):m(x;\eta)).
$$

The dual Legendre convex conjugate~\cite{LegendreIG-2010} $F(\theta)$ of $F^*(\eta)$  defined by the Legendre-Fenchel transform 
$$
F(\theta)=\max_\theta \{\inner{\theta}{\eta} -F^*(\eta)\}
$$ 
is
\begin{eqnarray}
F(\theta) &=& -\int_\calX p_0(x)\log m(x;\eta)\dmu(x),
\end{eqnarray}
This conjugate function can be interpreted  the {\em cross-entropy} $h^\times(p_0(x):m(x;\eta))$ between $p_0(x)$ and $m(x;\eta)$:
$$
h^\times(p_0(x):m(x;\eta)) \eqdef -\int_\calX p_0(x)\log m(x;\eta)\dmu(x). 
$$
The conjugate functions $F$ and $F^*$ are called the {\em potential functions} of dually flat space in information geometry.

A sanity check shows that $(F(\theta),F^*(\eta))$ is indeed a pair of  {\em convex conjugates} by verifying Young's inequality~\cite{Rockafellar-1970}
\begin{equation}
F(\theta)+F^*(\eta)=\inner{\theta}{\eta}.
\end{equation}

\begin{proof}
We have
\begin{eqnarray}
\inner{\theta}{\eta} &=& \sum_{i=1}^D \eta_i  \underbrace{\int_\calX (p_i(x)-p_0(x))\log m(x;\eta)\dmu(x)}_{\theta_i}, \label{eq:part1} \\ 
F(\theta)+F^*(\eta) &=&    \int (m(x;\eta)(x)-p_0(x)) \log m(x;\eta)\dmu(x).  \label{eq:part2}
\end{eqnarray}
Since $m(x;\eta)=\sum_{i=1}^D \eta_i p_i(x) + \eta_0 p_0(x)$, we have
$m(x;\eta)(x)-p_0(x)= \sum_{i=1}^D \eta_i  p_i(x)+(\eta_0-1) p_0(x)$, and $\eta_0-1=-\sum_{i=1}^D \eta_i$.
Thus Eq.~\ref{eq:part1} matches Eq.~\ref{eq:part2}.

Another proof consists in rewriting the differential entropy of the mixture as follows
\begin{eqnarray}
h(\eta) &=& \sum_{i=0}^{k-1} w_i \int p_i(x)\log\frac{1}{m(x;\eta)} \dmu(x),\\
&=& \sum_{i=1}^{k-1} w_i \underbrace{\int_\calX (p_i(x)-p_0(x))\log\frac{1}{m(x;\eta)}\dmu(x)}_{\theta_i} + \int p_0(x)\log\frac{1}{m(x;\eta)}\dmu(x),\\
&=& - \sum_{i=1}^{k-1} w_i\theta_i + h^\times(p_0(x):m(x;\eta)).
\end{eqnarray}
Since $w_i=\eta_i$ for $i\in [k-1]$, $F(\theta)=h^\times(p_0(x):m(x;\eta))$, and $F^*(\eta)=-h(m(x;\eta))$, we get Young' s inequality:
$$
F(\theta)+F^*(\eta)-\inner{\theta}{\eta}=0.
$$
\end{proof}

Function $F(\theta)$ is convex with respect to $\theta$, and the gradients of the convex conjugates are reciprocal, allowing one to convert (theoretically) from one coordinate system into the dual one: 
$$
\eta=\nabla F(\theta)\quad \mbox{and}\quad \theta=\nabla F^*(\eta).
$$

Let $\partial_i\eqdef \frac{\partial}{\partial\theta^i}$ and $\partial^i\eqdef \frac{\partial}{\partial\eta_i}$.
(Those subscript and superscript derivative notations emphasize the contravariant and covariant natures of  the derivations~\cite{Calin-2014,IG-2016}.)
At any tangent plane $T_p$ of $\calM$, the dual vector bases $\{\partial_i\}_i$ and $\{\partial^i\}_i$ satisfy
$\inner{\partial_i}{\partial^j}_p=\delta_i^j$ with $\delta_i^j=1$ iff $i=j$ and $0$ otherwise.
That is, the natural and moment coordinate systems are biorthogonal.

However, since neither $F$ or $F^*$ are usually available in closed forms in practice (except for the multinomial family that are $w$-mixtures with prescribed Dirac component distributions), those conversions are often computationally intractable.
See also the     log-linear models~\cite{Miura-2011} describing a pair of stochastic binary variables. 

\subsection{Kullback-Leibler of $w$-mixtures derived from the Bregman divergences or Legendre-Fenchel divergences}
Overall, it follows from the dually flat geometry of $\eta$-mixtures that the KL divergence between two $\eta$-mixture distributions
 of $\calM$ can be equivalently written as

\begin{eqnarray}
\KL(m_1:m_2)&=&  \int m(x;\eta_1) \log\frac{m(x;\eta_1)}{m(x;\eta_2)} \dmu(x),\nonumber\\
&=& B_{F^*}(\eta_1:\eta_2) = B_F(\theta_2:\theta_1),\\
&=& D_{F^*,F}(\eta_1:\theta_2) = D_{F,F^*}(\theta_2:\eta_1),
\end{eqnarray}
where $D_{F^*,F}(\eta_1:\theta_2)= F^*(\eta_1) + F(\theta_2)  -\inner{\eta_1}{\theta_2}$
denotes the {\em canonical divergence}~\cite{IG-2016} in dually flat spaces written using the mixed $\theta/\eta$-coordinate systems.

Let us check that $D_{F^*,F}(\eta:\theta')=\KL(m(x;\eta):m(x;\eta'))$.

\begin{proof} 
We have
$$
D_{F^*,F}(\eta:\theta')=\int (m(x;\eta)\log m(x;\eta)-p_0(x)\log m'(x;\eta')
-\sum_{i=1}^D \eta_i(p_i(x)-p_0(x))\log m(x;\eta') )\dmu(x)
$$

Since
 $$
\sum_{i=1}^D \eta_i p_i(x)\log m'(x;\eta')=(m(x;\eta)-\eta_0p_0(x))\log m(x;\eta'),
$$ 
we get:

\begin{eqnarray}
\lefteqn{D_{F^*,F}(\eta:\theta')=\int_\calX  \left(m(x;\eta)\log \frac{m(x;\eta)}{m(x;\eta')}-p_0(x)\log m(x;\eta)\right. }\nonumber\\
&& \left.+\eta_0p_0(x)\log m(x;\eta')+(1-\eta_0)p_0(x)\log m(x;\eta')\right)\dmu(x).
\end{eqnarray}

The proof highlights that the formula holds even for {\em unnormalized} mixture models~\cite{Amari-2009}, by writing the {\em extended KL divergence}~\cite{Bregman-2005}

$$
\KL(m(x;\eta):m(x;\eta')) = \int  \left(m(x;\eta)\log \frac{m(x;\eta)}{m(x;\eta')}
 + m(x;\eta') -m(x;\eta') \right)\dmu(x).
$$
When $\int m(x;\eta') \dmu(x) = \int m(x;\eta')\dmu(x) =1$, we recover the traditional KL divergence defined on probability measures.
\end{proof}

\begin{theorem}[KL of $w$-mixtures is a Bregman divergence]\label{lemma:KLetaBD}
The Kullback-Leibler divergence between two $\eta$-mixtures (or $w$-mixtures) is equivalent to a Bregman divergence defined for the convex Shannon information generator (negative entropy) on the $\eta$-parameters.
\end{theorem}

In practice, we may consider $w$-GMMs~\cite{comix-2016}, Gaussian Mixture Models  sharing the same components.
\begin{corollary}[KL of $w$-GMMS as a Bregman divergence]
The KL between Gaussian Mixture Models  sharing the same components ($w$-GMM~\cite{comix-2016}) is equivalent to a Bregman divergence.
\end{corollary}

The information geometry of $(\calM,\KL)$ is said {\em dually flat}~\cite{IG-2016} because the
 dual Christoffel symbol coefficients $\Gamma_{ijk}(x)$ and $\Gamma_{ijk}^*(x)$ have all their coefficients equal to zero~\cite{IGDiv-2010}.
Therefore geodesics are visualized as straight Euclidean lines in either the $\eta$- or $\theta$-affine coordinate systems.

Although many works addressed the exponential family manifolds and their curved subfamilies~\cite{IGChernoff-2013,kernelIG-2011,IG-2016}, only a very few papers study the mixture families (and curved mixture subfamilies~\cite{Anaya-2007,hayashi-2014}).
Note that the   multinomial family (the family of finite categorical distributions)  is both an exponential family~\cite{EF-2009} and a mixture family~\cite{IG-2016}.

Remark: The concept of {\em mixture family} holds beyond the statistical manifold setting~\cite{Calin-2014,IG-2016}.

Since $\nabla^2 F(\theta)=(\nabla^2 F^*(\eta))^{-1}$, we have
$\nabla^2 F(\theta)=(\nabla^2 F^*(\nabla F^*(\theta)))^{-1}$.
The Hessian $\nabla^2 F^*(\eta)$ corresponds to the Fisher Information Matrix (FIM) $I(\eta)$ which in the case of $w$-mixtures is guaranteed to be positive definite (and {\em never} degenerate). 
The FIM of a $k$-Gaussian mixture model may be degenerate (for example, when parameters are chosen so that two mixture components become identical), but the FIM of a $w$-GMM is never degenerated.
See also~\cite{Miura-2011} for mixed coordinate representation of mixtures that yields (in some cases) a diagonal FIM.

In general, computing the Shannon information of mixtures is computational intractable.
See for example, the case of a mixture of two Gaussians analyzed in~\cite{EntropyGMM-2008}.

By using the dual coordinate systems, the Jeffreys divergence of $w$-mixtures can be written without using explicitly the generator $F$:
$$
J(m(x;\eta);m(x;\eta'))=\int m(x;\eta-\eta')\log\frac{m(x;\eta)}{m(x;\eta')} \dmu(x) =\inner{\eta'-\eta}{\theta'-\theta}.
$$
By a slight abuse of notation, we wrote $m(x;\eta-\eta')$ as a shortcut of $m(x;\eta)-m(x;\eta')$.
However, we use implicitly the gradient of the generator to compute the natural parameter $\theta=\nabla F^*(\eta)$.

\begin{table}%
\centering
\begin{tabular}{|l|c|r|}\hline
 & Exponential Family & Mixture Family\\ \hline
Density & $p(x;\theta)=\exp(\inner{\theta}{x}-F(\theta))$ & $m(x;\eta)=\sum_{i=1}^{k-1}\eta_i f_i(x)+c(x)$\\
&  & $f_i(x)=p_i(x)-p_0(x)$\\
Family/Manifold & $\calM=\{p(x;\theta)\ :\ \theta\in\Theta^\circ \}$ &  $\calM=\{m(x;\eta)\ :\ \eta\in H^\circ \}$ \\
Convex function & $F$: cumulant & $F^*$: negative entropy\\
$\equiv ax+b$ & & \\
Dual coordinates & moment $\eta=E[x]$ &  $\theta^i=h^\times(p_0:m)-h^\times(p_i:m)$ \\
Fisher Information $g=(g_{ij})_{ij}$ & $g_{ij}(\theta)=\partial_i\partial_j F(\theta)$ & $g_{ij}(\eta)=\int_\calX \frac{f_i(x)f_j(x)}{m(x;\eta)}\dmu(x)$ \\ \hline
& & $g_{ij}(\eta)=-\partial_i\partial_j h(\eta)$ \\
Christoffel symbol & $\Gamma_{ij,k}=\frac{1}{2}\partial_i\partial_j\partial_k F(\theta)$ &
$\Gamma_{ij,k}=-\frac{1}{2}\int_\calX \frac{f_i(x)f_j(x)f_k(x)}{m^2(x;\eta)} \dmu(x)$ \\ \hline
Entropy & $-F^*(\eta)$ & $-F^*(\eta)$\\
Kullback-Leibler divergence & $B_F(\theta_2:\theta_1)$ & $B_{F^*}(\eta_1:\eta_2)$\\
& $=B_{F^*}(\eta_1:\eta_2)$ & $=B_F(\theta_2:\theta_1)$\\ \hline
skew Jensen divergence & skew Bhattacharrya div.~\cite{BR-2011} & skew Jensen-Shannon div. (\S\ref{sec:bdjd})\\ \hline
\end{tabular}
\caption{Characteristics of the dually flat geometries of Exponential Family Manifolds (EFMs) and Mixture Family Manifolds (MFMs).}
\label{tab:comparison}
\end{table}

Finally, let us emphasize that an integral-based Bregman generator $F_\calP$ on a probability distribution $p_\theta$ belonging to a family $\calP=\{p_\theta\}$ induces a parametric Bregman generator $F(\theta):=F_\calP(p_\theta)$.
The induced Bregman divergence $B_F$ amount to a {\em statistical divergence} $D_{\calP,F}$ between the parametric densities.
This statistical divergence can then be relaxed to arbitrary densities to get a statistical divergence $D_F$.
Appendix~\ref{sec:BDStatDiv} shows how to reconstruct the statistical divergences from integral-based Bregman generators for the exponential families (reverse KL divergence) and the mixture families (KL divergence).

\subsection{Application of $w$-mixtures: Optimal KL-averaging aggregation \label{sec:optint}}

Let us consider a {\em cluster} of $m$ machines $M_1,\ldots, M_m$ with the independently and identically sampled  data-set $\calO$ partitioned into $m$ pieces: $\calO_1,\ldots,\calO_m$ with
 $|\calO_i|=n_i$.
Dataset $\calO_i$ is stored {\em locally} in the   memory of machine $M_i$.
Liu and Ihler~\cite{DistributedEstimation-EF-2014} proposed 
\begin{enumerate}
\item to estimate the $m$ models $m(x;\lambda_i)$ locally (say, via Maximum Likelihood Estimators, MLEs, $\hat{\lambda}_i$'s on the local samples $\calO_i$), and then

\item merge/aggregate those local model estimates on a {\em central node} by performing {\em KL-averaging integration}:
$$
{\hat{\lambda}^\KL}  =  \arg\min_{\lambda} \sum_{i=1}^m  \KL(m(x:\hat{\lambda}_i):m(x:\lambda)).
$$

This KL-averaging integration has to be compared with the global MLE $\theta^\ML$ on the full data-set $\calO$.
The MLE is {\em equivariant} for a monotonic transformation $g$: $\widehat{g(\lambda)}_\MLE=g(\hat\lambda_\MLE)$.
\end{enumerate}

When the models belong to the same {\em exponential family}~\cite{EF-2009} (e.g., Gaussian models of the Gaussian family), they showed that the {\em KL-averaging} model integration yields {\em no information loss}:
Indeed, for exponential families~\cite{EF-2009} with log-density $\log p_F(x;\theta) =t(x)^\top \theta-F(\theta)$ (with $\theta$ the natural parameters, sufficient statistics $t(x)$ and $F(\theta)$ the log-normalizer or cumulant function), the dual moment parameter is $\eta=E[t(x)]=\nabla F(\theta)$.
We can convex the moment parameter $\eta$ into the corresponding natural parameter by $\theta=\nabla F^*(\eta)=\nabla F^{-1}(\eta)$.
In that case, the KL integration~\cite{DistributedEstimation-EF-2014}  yields for equally partitioned data-sets 

$$
{\hat{\theta}^\KL}  =  \nabla F^{-1}\left(\frac{1}{m} \sum_{i=1}^m \nabla F({\hat{\theta}_i})\right).
$$

That is, there is no information loss.
For an arbitrary partitioning of the data-set $\calO$,  the local MLEs are $\hat\eta_i=\nabla F(\hat\theta_i)=\frac{1}{n_i} \sum_j t(x_j)$, and the 
global MLE is $\hat\eta^\ML=\nabla F(\hat\theta^\ML)=\frac{1}{n} \sum_j t(x_j)=\frac{1}{n}(\sum_{i\in [m]} \sum_{x\in\calO_i} t(x))$.
That is, $\hat\eta^\ML=\frac{1}{n} \sum_i n_i\hat\eta_i^\ML$, or $\hat\theta^\ML=\nabla F^{-1}(\sum_i w_i \nabla F(\hat\theta_i^\ML))$, with $w_i=\frac{n_i}{n}$. 
Notice that aggregation of exponential family models requires to manipulate {\em explicitly} the log-normalizer $F(\theta)$ and its inverse gradient function $\nabla F^{-1}$, see~\cite{DistributedEstimation-EF-2014}.

The MLE for an exponential family is also characterized by a Bregman centroid~\cite{kmle-2012} using the exponential family/Bregman duality $\log p_F(x;\theta) =-B_{F^*}(t(x):\eta)+F^*(t(x))$ with $\eta=\nabla F(\theta)$:
$$
\max_\theta \sum_i \log p_F(x_i;\theta)  \equiv \min_\eta \sum_i B_{F^*}(t(x_i):\eta).
$$
In distributed estimation, the global MLE  of the $m$ datasets for an exponential family is thus obtained by performing the optimal KL-averaging integration when $n_1=\ldots=n_m$.
That is, $\hat\eta_\KL$ is the MLE of $\calO$.

Interestingly, they also report experiments on GMMs~\cite{DistributedEstimation-EF-2014} (\S 5.2) that are not exponential families with information loss, and stress out that the ``KL average still performs well as the global MLE'' on the MNIST data-set~\cite{DistributedEstimation-EF-2014}.

For $\eta$-mixtures (mixture families including the $w$-GMMs), the  KL-averaging integration~\cite{DistributedEstimation-EF-2014,Amari-2007} is defined by the following optimization problem:

\begin{eqnarray}
\hat\eta^\KL &=&\arg\min_\eta  \sum_{i=1}^m \KL(m(x;\hat\eta_i):m(x;\eta)),\\
&=& \arg\min_\eta \sum_{i=1}^m B_{F^*}(\hat\eta_i:\eta).
\end{eqnarray}

Since the {\em right-sided Bregman centroid}~\cite{Bregmancentroid-2009} is always the center of mass {\em whatever} the chosen Bregman generator\footnote{Here, it is specially interesting since $F^*$ is not available in closed form, and we bypass its use.}, we end up with the {\em optimal KL-average integration} (best parameter) for $\eta$-mixtures:
$$
{\hat{\eta}}^\KL  =  \frac{1}{m} \sum_{i=1}^m {\hat{\eta}_i}.
$$  
(or equivalently, ${\hat{w}}^\KL  =  \frac{1}{m} \sum_{i=1}^m {\hat{w}_i}$).


\begin{theorem}[Optimal KL-averaging integration of $w$-mixtures]
KL-averaging integration of $w$-mixtures can be performed optimally without information loss.
\end{theorem}

Note that the local model estimators may not be efficient for mixtures in general.
In fact, global Maximum Likelihood (ML) optimization tackles an untractable log-sum maximization for mixtures, and  
 the exact MLE solution for these mixtures maybe transcendental~\cite{MLEGMMTranscendental-2015}.

Notice that the KL-averaging integration does not depend on the local inference methods used: They can even be different methods on each machine.
For exponential families, it makes sense to use the MLE because of its link with a Bregman centroid on the expectation parameters.
It is interesting to characterize the information loss for curved mixture subfamilies~\cite{Anaya-2007,hayashi-2014} according to the notion of statistical curvature.

Similarly, we can cluster a set of $w$-mixtures (like $w$-GMMs) using $k$-means methods~\cite{differentialnormal-2007,clusteringnormal-2009} with respect to the KL divergence: 
To assign a $w$-mixture $m(x;\eta)$ to a cluster $w$-mixture prototype $m(x;\eta_j^c)$, we need to estimate $\KL(m(x;\eta):m(x;\eta_j^c))=B_{F^*}(\eta:\eta_j^c)$ (say, using Monte-Carlo stochastic estimation). Then the  $w$-mixture prototype of each cluster is updated by taking the centroid of the $\eta$-coordinates, and the process is repeated until (local) convergence~\cite{differentialnormal-2007,clusteringnormal-2009} .

We report how $w$-mixtures can be inferred efficiently in~\cite{mmix-2017}.

Note that when the components of $w$-mixture belong to the same exponential family, i.e.
 $m(x;w)=\sum_i w_i p(x;\theta_i^c)$, then we can find the best distribution~\cite{schwander2013} of that exponential family that simplifies the $w$-mixture as follows:

\begin{eqnarray}
\theta^c_\mathrm{opt} &=& \arg\min_{\theta^c} \KL(m(x;\eta) : p_{\theta^c}(x)),\\
&=& \nabla F^{-1}(\sum_i w_i \nabla F(\theta_i^c)).
\end{eqnarray}

\subsection{Skew Jensen-Shannon divergences of $w$-mixtures\label{sec:bdjd}}

Let the {\em skew $\alpha$-Jensen-Shannon divergence}~\cite{js-1991}  be defined by  
$$
\JS_\alpha(p:q) \eqdef (1-\alpha)\KL(p:m_\alpha)+\alpha\KL(q:m_\alpha),
$$
for the mixture  $m_\alpha(x)=(1-\alpha)p(x)+\alpha q(x)$ with $\alpha\in (0,1)$,
and  define the {\em $\alpha$-Jensen divergence}~\cite{Zhang-2004,BR-2011} by
$$
J_{F^*,\alpha}(\eta_1:\eta_2) \eqdef 
	(1-\alpha)F^*(\eta_1)+\alpha F^*(\eta_2)-F^*((1-\alpha)\eta_1+\alpha\eta_2),$$
for the Shannon information
	$F^*(\eta)= -h(m(x;\eta))$.
	
 We have in the limit cases~\cite{Zhang-2004,BR-2011} for $m_1(x)=m(x;\eta_1)$ and $m_2(x)=m(x;\eta_2)$:
	
	\begin{eqnarray*}
	\lim_{\alpha\rightarrow 1^-} \frac{1}{\alpha (1-\alpha)} J_{F^*,\alpha}(\eta_1:\eta_2) = B_{F^*}(\eta_1:\eta_2)=\KL(m_1:m_2)\\
	\lim_{\alpha\rightarrow 0^+} \frac{1}{\alpha (1-\alpha)} J_{F^*,\alpha}(\eta_1:\eta_2) = B_{F^*}(\eta_2:\eta_1)=\KL(m_2:m_1)
	\end{eqnarray*}
	
Since the combination of $w$-mixtures is a $w$-mixture, we have
$$
m_\alpha(x) \eqdef (1-\alpha)m(x;\eta_1)+\alpha m(x;\eta_2)= m(x;(1-\alpha)\eta_1+\alpha\eta_2).
$$

Plugging the Shannon negative entropy $h$ in $F^*(\eta)=-h(m(x;\eta))$, we get 
$$
J_{F^*,\alpha}(\eta_1:\eta_2)=   h(m_\alpha)-(1-\alpha)h(m_1)-\alpha h(m_2).
$$

We can rewrite 
\begin{eqnarray}
J_{F^*,\alpha}(\eta_1:\eta_2)&=& \int \left( -((1-\alpha)m_1(x)+\alpha m_2(x))\log m_\alpha(x)\right.\nonumber \\ 
&& \left. +(1-\alpha)m_1(x)\log m_1(x)+\alpha m_2(x) \log m_2(x) \right) \dmu(x)\\
 &=& \int \left( (1-\alpha)m_1(x)\log \frac{m_1(x)}{m_\alpha(x)} + \alpha m_2(x)\log \frac{m_2(x)}{m_\alpha(x)}  \right)\dmu(x),\\
&=& (1-\alpha)\KL(m_1:m_\alpha)+\alpha\KL(m_2:m_\alpha).
\end{eqnarray}

Thus we get
\begin{equation}
J_{F^*,\alpha}(\eta_1:\eta_2)=(1-\alpha)\KL(m_1:m_\alpha) + \alpha\KL(m_2:m_\alpha)=\JS_\alpha(m_1:m_2).
\end{equation}

In particular, when $\alpha=\frac{1}{2}$, $J_{F^*,\frac{1}{2}}(\eta_1:\eta_2)=\frac{1}{2}\JS(m_1:m_2)$ is the {\em Jensen-Shannon divergence}~\cite{js-1991},
and when $\alpha\rightarrow 1$, $\frac{1}{1-\alpha }J_{F^*,\alpha}(\eta_1:\eta_2)=\KL(m_1:m_2)$.

\begin{theorem}[$\alpha$-Jensen-Shannon div. equivalent to $\alpha$-Jensen div. for $w$-mixtures]
The $\alpha$-Jensen-Shannon divergences $\JS_\alpha(m(x;\eta_1):m(x;\eta_2))$ between two $\eta$-mixtures amount is equivalent to the $\alpha$-Jensen divergences 
 $J_{F^*,\alpha}(\eta_1:\eta_2)$ between their $\eta$-mixture parameters: $\JS_\alpha(m(x;\eta_1):m(x;\eta_2))=J_{F^*,\alpha}(\eta_1:\eta_2)$.
\end{theorem}

\begin{corollary}[KL as a limit case of skew Jensen-Shannon divergence]
In the limit case, we have $\lim_{\alpha\rightarrow 1} \frac{1}{\alpha(1-\alpha) }J_{F^*,\alpha}(\eta_1:\eta_2)= 
\lim_{\alpha\rightarrow 1} \frac{1}{\alpha(1-\alpha)}\JS_\alpha(m(x;\eta_1):m(x;\eta_2))=
\KL(m_1:m_2)$. 
Similarly, we have $\lim_{\alpha\rightarrow 0} \frac{1}{\alpha(1-\alpha) }J_{F^*,\alpha}(\eta_1:\eta_2)= 
\lim_{\alpha\rightarrow 0} \frac{1}{\alpha(1-\alpha)}\JS_\alpha(m(x;\eta_1):m(x;\eta_2))=
\KL(m_2:m_1)$.
\end{corollary}

\begin{figure}
\centering
\includegraphics[width=\textwidth]{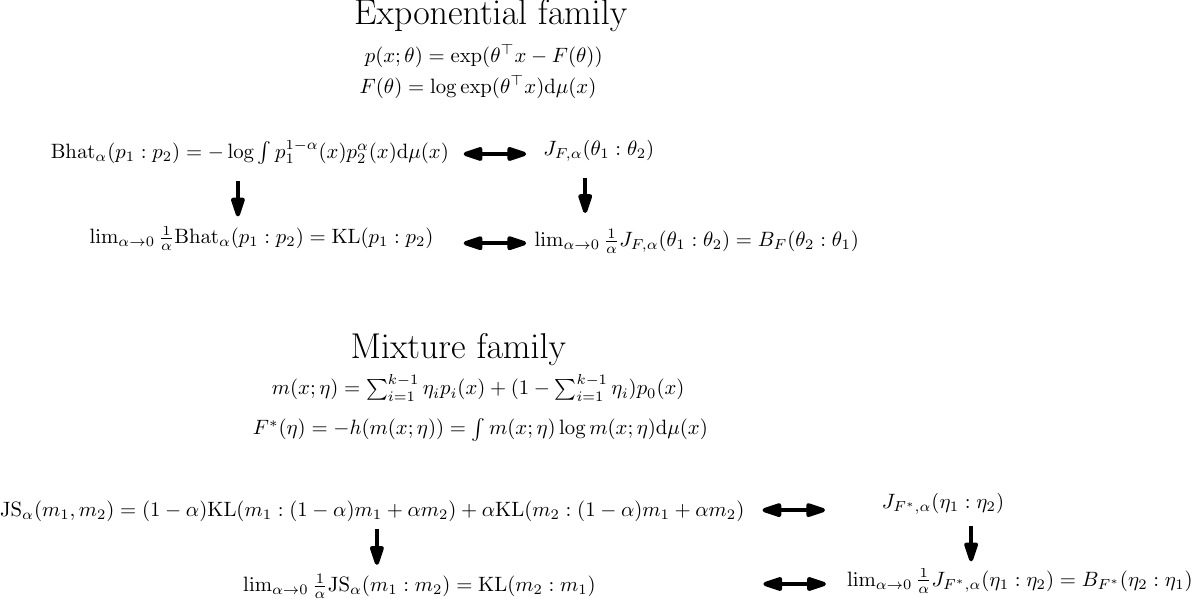}
\caption{Statistical divergences and corresponding parametric divergences: Comparisons of the exponential family with the mixture family.\label{fig:dist}}
\end{figure}

For exponential families, the skew Bhattacharrya divergences are shown to be equivalent to skew Jensen divergences~\cite{BR-2011}.
Figure~\ref{fig:dist} displays the relationships of the different statistical divergences with the equivalent parametric divergences.
Let us remark that we can approximate the Bregman divergence $B_F$ by a skew $\alpha$-Jensen divergence for small enough $\alpha>0$: This allows one to bypass
 the need to compute the gradient $\nabla F$.

\section{On mixture closures and divergences\label{sec:closure}}

\subsection{Upper bound on the KL divergence between arbitrary mixtures}

Let $m(x)\eqdef \sum_{i=0}^{k-1} w_i p_i(x)$ and $m'(x)\eqdef \sum_{i=0}^{k-1} w_i' p_i'(x)$ be two mixtures, both with exactly $k$ components.
We have the following upper bound~\cite{Goldberger-2005}:
	$$
	\KL(m:m') \leq \KL(w:w')+ \sum_{i=0}^{k-1} w_i \KL(p_i:p_i')
	$$
The bound can be strengthened by taking all $k!$ permutations $\sigma$:
$$
\KL(m:m') \leq \min_{\sigma}  \left\{ \KL(\sigma(w):w')+\sum_{i=0}^{k-1} w_{\sigma(i)} \KL(p_{\sigma(i)}:p_i')  \right\}
$$
	
The proof relies on the   {\em log-sum inequality}:
\begin{lemma}[log-sum inequality~\cite{Cover-2012}]
Given two finite positive number sequences $A\eqdef\{a_i\}_{i=0}^{k-1}$ and $B\eqdef\{b_i\}_{i=0}^{k-1}$ with $a\eqdef\sum_{i=0}^{k-1} a_i$ and $b\eqdef\sum_{i=0}^{k-1} b_i$.
 We have $\sum_{i=0}^{k-1} a_i\log\frac{a_i}{b_i} \geq a\log\frac{a}{b}$.
\end{lemma}

We shall prove a stronger {\em convex-sum inequality}:
\begin{lemma}[convex-sum inequality~\cite{Csiszar-2004}]
Given two finite positive number sequences $A\eqdef\{a_i\}_{i=0}^{k-1}$ and $B\eqdef\{b_i\}_{i=0}^{k-1}$ with $a\eqdef\sum_{i=0}^{k-1} a_i$ and $b\eqdef\sum_{i=0}^{k-1} b_i$, and $f$ a 
convex function. We have:
\begin{equation*}\label{eq:lsif}
\sum_{i=0}^{k-1} a_i f\left(\frac{b_i}{a_i}\right) \geq a f\left(\frac{b}{a}\right).
\end{equation*}
\end{lemma}

This later lemma generalizes the log-sum inequality obtained for $f(u)=u\log u$, a strictly convex function (and swapping role $A\leftrightarrow B$),
or for $f(u)=-\frac{1}{u}\log u$ which is a strictly convex on $u>0$.

\begin{proof}
Recall the {\em Jensen (discrete) inequality}~\cite{Niculescu-2006} for a convex function $f$:
	$$
	\sum_{i=0}^{k-1} w_i f(x_i) \geq  f\left(\sum_{i=0}^{k-1} w_i x_i\right),\quad w\in \Delta_k^\circ
	$$
	
	Let $w_i\eqdef \frac{a_i}{a}$ and $x_i\eqdef\frac{b_i}{a_i}$. Then it comes that
	$$
	\sum_{i=0}^{k-1} \frac{a_i}{a} f\left(\frac{b_i}{a_i}\right)   \geq  f\left(\sum_{i=0}^{k-1} \frac{\cancel{a_i}}{a}\frac{b_i}{\cancel{a_i}}\right).
	$$
	Finally, we get the convex-sum inequality
	$$
	\sum_{i=0}^{k-1}  a_i f\left(\frac{b_i}{a_i}\right) \geq a f\left(\frac{b}{a}\right).
	$$
	\end{proof}

Now we are ready to prove $\KL(m:m') \leq \KL(w:w')+\sum_{i=0}^{k-1} w_i \KL(p_i:p_i')$ using the log-sum inequality.
Let $a_i\eqdef w_i p_i(x)$ and $b_i \eqdef w_i'p_i'(x)$ so that $a=m(x)$ and $b=m'(x)$:

\begin{eqnarray*} 
\sum_{i=0}^{k-1} a_i\log\frac{a_i}{b_i} &\geq& a\log\frac{a}{b},\\
\sum_{i=0}^{k-1} w_i p_i(x)\log\frac{w_i p_i(x)}{w_i'p_i'(x)} &\geq& m(x)\log\frac{m(x)}{m'(x)},
\end{eqnarray*}  
Integrating over the support $\calX$, we get 
\begin{eqnarray*} 
\sum_{i=0}^{k-1} w_i \log\frac{w_i}{w_i'} \underbrace{\int p_i(x)\dmu(x)}_{=1} + w_i\KL(p_i:p_i')  &\geq& \KL(m:m'),\\
\KL(w:w') + \sum_i w_i \KL(p_i:p_i')  &\geq&  \KL(m:m').
\end{eqnarray*}  

In particular, we have for $w$-mixtures: 
\begin{equation}
\KL(m:m')\leq \KL(w:w').
\end{equation}
Furthermore, it holds that
$$
\KL(w:w')\leq \log \frac{\max_i w_i}{\min_i w_i'}\leq -\log \min_i w_i'.
$$
Note that the upper bound may tend to infinity when $\min_i w_i'\rightarrow\infty$.

\subsection{Divergence inequalities for $w$-mixtures}

\begin{theorem}[Upper bound on $f$-divergences of $w$-mixtures]
The $f$-divergence $I_f(m(x;w) : m(x;w'))$ between any two $w$-mixtures is upper bounded by $I_f(w:w')=\sum_{i=0}^{k-1} w_i f(\frac{w_i'}{w_i})$.
\end{theorem}

\begin{proof}
We use a generalization of the log-sum inequality  to any convex function $f$ (see~\cite{Csiszar-2004}, p. 448):
For two finite positive number sequences $A\eqdef\{a_i\}_{i=0}^{k-1}$ and $B\eqdef\{b_i\}_{i=0}^{k-1}$, we have
$\sum_i a_i f\left(\frac{b_i}{a_i}\right) \geq a f\left(\frac{b}{a}\right)$.
It follows that 
$m(x;w)f\left(\frac{m(x;w')}{m(x;w)}\right) \leq \sum_{i=0}^{k-1} w_i   p_i(x)  f\left( \frac{w_i' p_i(x)  }{w_i p_i(x) }\right)=  \sum_{i=0}^{k-1}  w_i  f\left(\frac{w_i' }{w_i}\right)p_i(x)$
Carrying out integration on the support $\calX$, we get
$I_f(m(x;w) : m(x;w'))\leq   I_f(w:w')$ since $\int_\calX p_i(x)\dmu(x)=1$.
Notice that the KL divergence is a $f$-divergence obtained for the generator $f(u)=-\log u$.
\end{proof}

For the KL divergence of $w$-mixtures, we thus have:
\begin{equation}
\KL(m(x;w) : m(x;w'))\leq \KL(w:w').
\end{equation}

The  KL divergence can be extended to positive measures $\tp$ and $\tq$ (not necessarily normalized) by:
$$
\KL(\tp : \tq)=\int_\calX \left( \tp(x)\log \frac{\tp(x)}{\tq(x)} + \tq(x)- \tp(x) \right) \dmu(x).
$$
For non-normalized $w$-mixtures $m(x;\tw')$ and $m(x;\tw)$, it comes that 
$$
\KL(m(x;\tw') : m(x;\tw)) \leq \KL(\tw:\tw')=\sum_i w_i\log\frac{w_i}{w_i'} + w_i'-w_i.
$$

Let $\tm(x)=m(x;\tw)$ and $\tm'(x)=m(x;\tw')$.
In fact, when $\tm'(x)=\lambda \tm(x)$ (for $\tw'=\lambda \tw$), we have
$$
\KL(m(x;\tw') : m(x;\tw)) = (\lambda-1)-\log\lambda.
$$
So in that particular case, we can compute in closed form the extended KL between two unnormalized GMM models (in that case, $w$-GMMs).

We can upper bound $\KL(w:w')$ using the maximum and the minimum positive weights as follows
$$
\KL(w:w') \leq \log\frac{\max_i w_i}{\min_i w_i'}.
$$

Thus when the minimum weight of $w$-mixtures is at least $\eps>0$ and at most $\frac{1}{k}$ (fat $w$-mixtures), 
we have $I_f(w:w')=\sum_{i=0}^{k-1} w_i\log\frac{w_i}{w_i'}\leq \sum_{i=0}^{k-1} w_i\log {w_i} - \log\eps \leq -\log \eps$ since $w_i\leq 1$.
That is, the $f$-divergence of fat $w$-mixtures is upper bounded.

In general, the discrete $f$-divergence can be upper bounded by $\max_i f\left(\frac{w_i'}{w_i}\right)$ so that
$$
I_f(m(x;w) : m(x;w')) \leq I_f(w:w') \leq \max_i f\left(\frac{w_i'}{w_i}\right).
$$

Let us report  lower and upper bounds for the KL divergence between $w$-mixtures using the mediant inequalities~\cite{Mediant-1990}.
Let $[k]=\{1,\ldots, k\}$.

\begin{proposition}[Mediant inequalities]
Let $n_1,\ldots, n_k$ and $m_1,\ldots,m_k$ be two sequences of positive reals.
Then the following inequalities hold:
$$
\min_{i\in [k]} \frac{n_i}{m_i} \leq \frac{\sum_{i=1}^k n_i}{\sum_{i=1}^k  m_i}\leq \max_{i\in [k]} \frac{n_i}{m_i}.
$$
\end{proposition}
These lower and upper inequalities are called the mediant inequalities because given $\frac{a}{b}<\frac{c}{d}$ for positive $a,b,c,d>0$, the mediants of these fractions is the fraction $\frac{a+b}{c+d}$ which satisfies $\frac{a}{b}<\frac{a+b}{c+d}<\frac{c}{d}$. 
(The proof   dates back to Cauchy in 1816 who studied Farey sequences~\cite{svalbe2003farey}.)  

By considering positive weights $w_i$'s such that $n_i=w_ia_i$ and $m_i=w_ib_i$, we get the weighted mediant inequality:

\begin{proposition}[Weighted mediant inequality]
Let $a_1,\ldots, a_k$, $b_1,\ldots,b_k$, and $w_1,\ldots, w_k$ be positive reals.
Then we have:
$$
\min_{i\in [k]} \frac{a_i}{b_i} \leq \frac{\sum_{i=1}^k w_ia_i}{\sum_{i=1}^k  w_ib_i}\leq \max_{i\in [k]} \frac{a_i}{b_i}.
$$
\end{proposition}

Since 
$\frac{\sum_{i=1}^k n_i}{\sum_{i=1}^k  m_{\sigma(i)}}=\frac{\sum_{i=1}^k n_i}{\sum_{i=1}^k  m_i}$, we may consider any permutation $\sigma$ for the denominator and get:
$$
\min_{i\in [k]} \frac{n_i}{m_{\sigma(i)}} \leq  \frac{\sum_{i=1}^k n_i}{\sum_{i=1}^k  m_i} \leq \max_{i\in [k]} \frac{n_i}{m_{\sigma(i)}}.
$$

By taking the maximum lower bound and minimum lower bound with respect to all $k!$ permutations, we get the following improved mediant inequality:

$$
\max_\sigma \min_{i\in [k]} \frac{n_i}{m_{\sigma(i)}} \leq  \frac{\sum_{i=1}^k n_i}{\sum_{i=1}^k  m_i} \leq \min_\sigma\max_{i\in [k]} \frac{n_i}{m_{\sigma(i)}}.
$$

\begin{proposition}[Improved Mediant inequality]
Let $n_1,\ldots, n_k$ and $m_1,\ldots,m_k$ be  positive reals.
Then we have:
$$
\frac{\min_{i\in [k]} n_i}{\max_{i\in [k]} m_i} \leq \max_\sigma \min_{i\in [k]} \frac{n_i}{m_{\sigma(i)}} \leq  \frac{\sum_{i=1}^k n_i}{\sum_{i=1}^k  m_i} \leq \min_\sigma\max_{i\in [k]} \frac{n_i}{m_{\sigma(i)}}\leq \frac{\max_{i\in [k]} n_i}{\min_{i\in [k]} m_i}.
$$
\end{proposition}

Now, consider two $w$-mixtures $m(x)=\sum_{i=1}^k w_ip_i(x)$ and $m'(x)=\sum_{i=1}^k w_i'p_i(x)$ of a mixture family.
Then using the weighted mediant inequality (with the weights $\alpha_i=p_i(x)$), we have:
$$
\min_{i\in[k]}\frac{w_i}{w_i'} \leq \frac{\sum_{i=1}^k w_ip_i(x)}{\sum_{i=1}^k w_i'p_i(x)}\leq \max_{i\in[k]}\frac{w_i}{w_i'}
$$
Therefore, we get the following proposition:
\begin{proposition}[Bounds for the KLD between $w$-mixtures]
$$
\log \min_{i\in[k]}\frac{w_i}{w_i'} \leq \KL[m:m'] \leq \log\max_{i\in[k]}\frac{w_i}{w_i'}
$$
\end{proposition}
Notice that the $w$-mixtures can be multivariate here, and that $\KL[m_1:m_2]\leq \KL[w:w']\leq \log\max_{i\in[k]}\frac{w_i}{w_i'}$.
Thus we have:
$$
\log \min_{i\in[k]}\frac{w_i}{w_i'} \leq \KL[m:m'] \leq \KL[w:w']\leq  \log\max_{i\in[k]}\frac{w_i}{w_i'}
$$

For the KLD between two univariate mixtures of $k$ components (not necessarily with the same components), we may also use the improved mediant inequality so that we have:
$$
\max_\sigma  \min_i \frac{w_ip_i(x)}{w_{\sigma(i)}'p_{\sigma(i)}'(x)} \leq \frac{\sum_{i=1}^k w_ip_i(x)}{\sum_{i=1}^k w_i'p_i'(x)}\leq 
\min_\sigma \max_i \frac{w_ip_i(x)}{w_{\sigma(i)}'p_{\sigma(i)}'(x)}. 
$$

Therefore we get:
$$
\int m_1(x) \max_\sigma  \min_i \frac{w_ip_i(x)}{w_{\sigma(i)}'p_{\sigma(i)}'(x)}\dx\leq \KL[m_1:m_2]\leq 
\int m_1(x) \min_\sigma \max_i \frac{w_ip_i(x)}{w_{\sigma(i)}'p_{\sigma(i)}'(x)}\dx.
$$

The extrema $\max_\sigma  \min_i \frac{w_ip_i(x)}{w_{\sigma(i)}'p_{\sigma(i)}'(x)}$ 
and $\min_\sigma \max_i \frac{w_ip_i(x)}{w_{\sigma(i)}'p_{\sigma(i)}'(x)}$ change combinatorially at discrete locations of the support so that we can partition the support of the mixtures into elementary intervals. 
Then for each elementary interval, we can calculate the definite elementary integrals using the CDF of the univariate component distributions, similar to~\cite{lsemixture-2016}.

Define the {\em $\alpha$-divergence}~\cite{Amari-2009} for $\alpha\in\bbR\backslash\{0,1\}$ by
$$
I_\alpha(p:q) \eqdef \frac{1}{\alpha(1-\alpha)} \left( 1-\int_\calX p^\alpha(x)q^{1-\alpha}(x) \dmu(x)\right) = I_{-\alpha}(q:p),
$$
with $\lim_{\alpha\rightarrow 1} I_\alpha(p:q)=\KL(p:q)$ and $\lim_{\alpha\rightarrow 0} I_\alpha(p:q)=\KL(q:p)$ (reverse Kullback-Leibler divergence).
We can define the $\alpha$-divergences according to the {\em Chernoff $\alpha$-coefficient}~\cite{Chernoff-2011}: 
$$
c_\alpha(p:q) \eqdef \int_\calX p^\alpha(x)q^{1-\alpha} \dmu(x)= \int q(x) \left(\frac{p(x)}{q(x)}\right)^{\alpha} \dmu(x).
$$
We have $I_\alpha(p:q) = \frac{1}{\alpha(1-\alpha)} (1-c_\alpha(p:q))$.
	
For $w$-mixtures, the following inequalities hold:
$$
\frac{\lambda}{\Lambda'} \leq \frac{p(x)}{q(x)} \leq \frac{\Lambda}{\lambda'},
$$ 
with $\lambda=\min_i w_i\leq\frac{1}{k}$, $\lambda'=\min_i w_i'\leq\frac{1}{k}$, $\Lambda=\max_i w_i\geq\frac{1}{k}$ and $\Lambda'=\max_i w_i'\geq\frac{1}{k}$.
	
Since function $x^\alpha$ is increasing when $\alpha> 0$ and decreasing  when $\alpha< 0$ (for $x>0$), we get
	\begin{eqnarray*}
	   \left(\frac{\lambda}{\Lambda'}\right)^{\alpha}   &\leq& c_\alpha(p:q)  \leq  \left(\frac{\Lambda}{\lambda'}\right)^{\alpha},\quad \alpha> 0,\\
		 \left(\frac{\Lambda}{\lambda'}\right)^{\alpha}   &\leq& c_\alpha(p:q)  \leq  \left(\frac{\lambda}{\Lambda'}\right)^{\alpha},\quad \alpha< 0.\\
	\end{eqnarray*}
	
It follows bounds on $\alpha$-divergences and related R\'enyi and Tsallis divergences~\cite{RT-2011} according to the mixture weight extrema.

\subsection{On $w$-mixture closures}
The manifold $\calM$ of $w$-mixtures is parameterized by the {\em open} probability simplex $\Delta_{k-1}^\circ$.
When topologically closing the manifold $\calM$, we consider $\bar\Delta_{k-1}$.
Take a $l$-face of the $(d-1)$-dimensional simplex $\Delta_{k-1}^\circ$.
When $l>0$, the sub-simplex $\sigma\in \bar\Delta_{k-1}$ is a $l$-dimensional simplex, and $\sigma^\circ$ parameterizes a $w$-mixture family of order $l>0$. 
In the extreme case, we consider order-1 $w$-mixture induced by a simplex edge $\sigma_1\in\Delta_{k-1}^\circ$ with extremity component distributions $p$ and $q$.
For example, distributions $p$ and $q$ can be Gaussian mixture models.
Define $m^\epsilon(p,q)=(1-\epsilon)p+\epsilon q=p+\epsilon (q-p)=m^{1-\eps}(q:p)$ for $\epsilon\in [0,1]$.
In the limit cases, the $w$-mixtures $m^\epsilon$ yields (with $w\in\Delta_1^\circ$):
$\lim_{\epsilon\rightarrow 0} m^\epsilon(p,q) =\lim_{\epsilon\rightarrow 1} m^\epsilon(q,p) =p$ and
$\lim_{\epsilon\rightarrow 1} m^\epsilon(p,q)=\lim_{\epsilon\rightarrow 0} m^\epsilon(q,p)=q$
Let $I_f^\eps(p:q) \eqdef I_f(m^\epsilon(p,q),m^\epsilon(q,p))$.
How far is $I_f^\eps(p:q)$ from its closure $I_f(p:q)$?

On one hand, we have the following theorem:
\begin{theorem}[Total variation continuity]
We have the following identity:
\begin{equation}
\TV^\eps(p,q)= |1-2\epsilon| \TV(p,q),
\end{equation}
since $m^\epsilon(p,q)-m^\epsilon(q,p)=(1-2\epsilon)(p-q)$.
\end{theorem}
Thus  $\lim_{\eps\rightarrow 0}\TV^\eps(p,q)=\lim_{\eps\rightarrow 1}\TV^\eps(p,q)=\TV(p,q)$.

On the over hand, $\KL^\epsilon(p:q) \eqdef \KL(m^\epsilon(p,q):m^\epsilon(q,p))$ has been shown to amount to a Bregman divergence. 
That is, $\KL^\eps(p:q)=B_{F^*}(\eps:1-\eps)$ for 1D generator $F^*(\eta)=\int_\calX (p(x)+\eta(q(x)-p(x)))\log(p(x)+\eta(q(x)-p(x)))\dmu(x)$.
By using the fact that the Bregman divergence is the tail of a first-order Taylor expansion~\cite{IG-2016}, we get using Lagrange exact reminder: 
$\KL^\eps(p:q) = \frac{1}{2}(1-2\epsilon)^2 (F^*)''(\eta)$ for $\eta\in [\epsilon,1-\epsilon]$.
However, the KL between $p$ and $q$ may potentially be infinite so that in general $\forall \eps\not =0, \KL^\epsilon(p:q)\not = \KL(p:q)$.

Using the {\em joint convexity} of the KL divergence, we can show that
$\KL^\eps(p:q)  \leq  \KL(p:q)+ \epsilon^2 J(p;q)$, where $J$ is Jeffreys divergence.

Let us relate the $f$-divergence between the 1D $\eta$-mixture and its extremities (closure) as follows:

\begin{theorem}[$f$-divergence inequalities]
We have the following inequalities:
\begin{eqnarray}
I_f^\eps(p:q) &\leq& (1-\eps)I_f(p:q) + \eps I_f(q:p), \label{eq:first}\\
I_f^\eps(p:q) &\leq&  (1-\eps) f\left(\frac{\eps}{1-\eps}\right) + \eps  f\left(\frac{1-\eps}{\eps}\right). \label{eq:second}
\end{eqnarray}
When $I_f$ is symmetric ($f=f^\diamond$), $I_f^\eps(p:q) \leq I_f(p:q)$.
That is, mixing distributions decrease symmetrized $f$-divergence values.
\end{theorem}

\begin{proof}
Apply the convex-sum inequality on
$A \eqdef \{(1-\eps)p(x),\eps q(x)\}$ and $B \eqdef \{(1-\eps)q(x),\eps p(x)\}$, so that  $a=m^\eps(p,q)$ and $b=m^\eps(q,p)$.
First, let $a_0\eqdef (1-\eps)p(x)$, $b_0\eqdef (1-\eps) q(x)$, and $a_1\eqdef \eps q(x)$ and $b_1\eqdef \eps p(x)$.
	We get Ineq.~\ref{eq:first}.
Second, let $a_0\eqdef (1-\eps)p(x)$, $b_0\eqdef \eps p(x)$, and $a_1\eqdef \eps q(x)$ and $b_1\eqdef (1-\eps)q(x)$.
We get Ineq.~\ref{eq:second}.
Note that when $\eps\rightarrow 0$, the second rhs inequality yields $f(0)+0f(\infty)$, similar to $I_f\leq f(0)+\frac{f(\infty)}{\infty}$ of~\cite{Vajda-1987}.
\end{proof}

We can also bound $\KL^\eps(p:q)$ for $\eps\in(0,1)$ as follows:
Let $C_\eps=\max\{\eps,1-\eps\}$ and $c_\eps=\min\{\eps,1-\eps\}$.

\begin{eqnarray*}
\frac{C_\eps (p(x)+q(x))}{c_\eps (p(x)+q(x))}  &\geq& \frac{m^\eps(p,q)(x)}{m^\eps(q,p)(x)} \geq \frac{c_\eps (p(x)+q(x))}{C_\eps (p(x)+q(x))},\\
\frac{C_\eps}{c_\eps}  &\geq& \frac{m^\eps(p,q)(x)}{m^\eps(q,p)(x)} \geq \frac{c_\eps}{C_\eps}>0,
\end{eqnarray*}

Thus since $\KL^\eps(p:q)=\int_\calX m^\eps(p,q)(x)\log\frac{m^\eps(p,q)(x)}{m^\eps(q,p)(x)} \dmu(x)$, we have:

\begin{equation}
 \log \frac{C_\eps}{c_\eps}\geq \KL^\eps(p:q) \geq \log \frac{c_\eps}{C_\eps},\quad \forall \eps\in (0,1)
\end{equation}
since $\int_\calX m^\eps(p,q)(x)\dmu(x)=1$.

We conclude with this theorem:

\begin{theorem}[KL of $\eps$-mixtures is a Bregman divergence]
For any pair of distributions $(p,q)$ and {\em any} $\epsilon>0$ there exist an $\epsilon$-close pair $(p^\eps=m^\eps(p,q),q^\eps=m^\eps(q,p))$ (wrt to total variation) such that
$\KL(p^\eps:q^\eps)$ amount to compute a Bregman divergence.
\end{theorem}

\begin{proof}
The total variation is bounded by $1$, and 
$\TV(p,p^\eps)=\frac{1}{2}\int |p(x)-(1-\eps)p(x)-\eps q(x)| \dmu(x)=\eps \TV(p,q)\leq \eps$.
Since $p^\eps$ and $q^\eps$ are $1$-mixtures, it follows that their Kullback-Leibler divergence corresponds to a Bregman divergence.
\end{proof}

We can obtain a lower bound on the total variation between two $w$-mixtures as follows:
First, consider the inequality $|a-b|\geq |\ |a|-|b|\ |$.
This inequality is easily checked by squaring both sides: $(a-b)^2=a^2+b^2-2ab\geq |\ |a|-|b|\ |^2=a^2+b^2-2|a|\ |b|$ which holds
because $ab\leq |a|\ |b|$.
Thus we have:

\begin{eqnarray}
\TV\left(\sum w_i f_i,\sum w_i' f_i\right) &=& \frac{1}{2} \left| \int (\sum_{i\in I} (w_i-w_i') f_i  - \sum_{i\in [D]\backslash I} (w_i-w_i') f_i) \dmu(x) \right|,\\
&=& \frac{1}{2} \left|   \sum_{i\in I} (w_i-w_i')   - \sum_{i\in [D]\backslash I} (w_i-w_i')  \right| \geq 0,
\end{eqnarray}
where $I$ is the set of indices such that $w_i\geq w_i'$ and $[D]=\{1,\ldots, D\}$.
Thus we have:

\begin{theorem}[Lower bound on the TV between two $w$-mixtures]
The total variation between two $w$-mixtures with fixed components $\{f_i\}$ is lower bounded as follows:
\begin{equation}
\TV\left(\sum w_i f_i,\sum w_i' f_i\right) \geq \frac{1}{2}  \left|  \left|\sum_{i\in I} (w_i-w_i')\right| - 
\left|\sum_{i\in [D]\backslash I} (w_i-w_i')\right|  \right| .
\end{equation}
\end{theorem}

Notice that when the two mixtures $m_1$ and $m_2$ with $k_1$ and $k_2$ components do not share common components, we can view these mixtures as the close of $w$-mixtures for $k=k_1+k_2$ components. The total variation distance has then lower bound $|1-1|=0$.
When the two mixtures are categorical distributions with $k$ components (i.e., bins), we have
\begin{equation}
\TV\left(\sum w_i \delta_{x_i}(x),\sum w_i' \delta_{x_i}(x)\right) = \frac{1}{2} \sum_{i=1}^k |w_i -w_i'| \geq \frac{1}{2}  \left|  \left|\sum_{i\in I} (w_i-w_i')\right| - 
\left|\sum_{i\in [D]\backslash I} (w_i-w_i')\right|  \right| .
\end{equation}

\section{Summary and conclusion}

Let us wrap-up and summarize our contributions as follows: 
We prove that the Kullback-Leibler (KL) divergence between two $w$-mixtures is equivalent to a Bregman divergence for the Bregman convex generator set to the Shannon negentropy (also called Shannon information). 
The induced geometry is a  dually flat manifold in information geometry~\cite{IG-2016,EIG-2018} callled the mixture family manifold.
It follows that the  KL-averaging integration~\cite{DistributedEstimation-EF-2014} of $w$-mixtures can be done optimally:
 This is useful for distributed estimations of $w$-mixtures.
We proved that the $\alpha$-Jensen-Shannon divergences between $w$-mixtures is equivalent to $\alpha$-Jensen divergences on their parameters.
This contrasts with the fact the  $\alpha$-Bhattacharrya divergence between two members of the same exponential family amounts 
to  $\alpha$-Jensen divergences~\cite{BR-2011}. Finally,  
we proved inequalities for the $f$-divergences of $w$-mixtures.

Note that MLE estimation of $w$-mixtures bears similarity with estimation of Cauchy parameters since it involves high-degree polynomial root solving~\cite{cauchymle-1970}.
Efficient inference of $w$-mixtures is studied in a forthcoming paper~\cite{mmix-2017}.

\vskip 0.5cm
A Java\texttrademark{} package for reproducible research implementing $w$-Gaussian Mixture Models ($w$-GMMs) is available at the following home page:\\
\centerline{\url{https://franknielsen.github.io/w-mixtures/}}

\section*{Acknowledgments}

The authors would like to thank  Ga\"etan Hadjeres for carefully reading a preliminary draft.
 
\bibliographystyle{plain}
\bibliography{wmixturefamilyV3BIB}

\appendix

\section{Common statistical distances}\label{sec:commondist}

The table below summarizes the common divergences met in information theory and statistical processing~\cite{js-1991}:
\vskip 0.5cm

\begin{tabular}{ll}
Distance name & formula\\ \hline\hline
Total variation & $\TV(p(x):q(x)) =\frac{1}{2}\int |p(x)-q(x)|\dmu(x)$\\
Kullback-Leibler divergence & $\KL(p(x):q(x)) = \int p(x)\log\frac{p(x)}{q(x)}\dmu(x)$\\
Jeffreys divergence & $J(p(x):q(x)) =\int (p(x)-q(x))\log\frac{p(x)}{q(x)}\dmu(x)$\\
& $= \KL(p(x):q(x)) +\KL(q(x):p(x))$\\
Lin $K$ divergence &  $K(p(x):q(x)) = \int (p(x))\log\frac{2p(x)}{p(x)+q(x)}\dmu(x)$\\
Jensen-Shannon divergence & $\JS(p(x):q(x))= \frac{1}{2} ( K(p(x):q(x)) + K(q(x):p(x)) )$ \\ \hline
\end{tabular}
\vskip 0.5cm
Furthermore, we have the following inequalities~\cite{js-1991}: $K\leq\frac{1}{2}\KL$ (and $K\leq 1$), $\JS\leq\frac{1}{4}J$, and $\JS\leq 2\TV$ (and $\TV\leq 1$).

\section{Extended $f$-divergences and Monte Carlo estimations}\label{sec:extendfdiv}


We define the  {\em extended $f$-divergences} as follows:

\begin{definition}[Extended $f$-divergence]
The extended $f$-divergence for a convex generator $f$, strictly convex at $1$ and satisfying $f(1)=0$ is defined by 
$$
I_f^e(p:q)=\int p(x)\left( f\left(\frac{q(x)}{p(x)}\right)-f'(1)\left(\frac{q(x)}{p(x)}-1\right)   \right)\dmu(x).
$$
\end{definition}

For a strictly convex generator $f$, let us consider the scalar Bregman divergence~\cite{BD-1967}:
\begin{equation}\label{eq:bds}
B_f(a:b)= f(a)-f(b)-(a-b)f'(b)\geq 0.
\end{equation}

Setting $a=\frac{q(x)}{p(x)}$ and $b=1$ in Eq.~\ref{eq:bds}, and using the fact that $f(1)=0$, we get
$$
f\left(\frac{q(x)}{p(x)}\right)-\left(\frac{q(x)}{p(x)}-1\right)f'(1)\geq 0.
$$

Therefore we define the {\em extended $f$-divergences} as 
\begin{equation}
I_f^e(p:q) = \int p(x) B_f\left(\frac{q(x)}{p(x)}:1\right)\dmu(x)\geq 0.
\end{equation}
That is, the formula for the extended $f$-divergences is
\begin{equation}
I_f^e(p:q)= \int p(x) \left( f\left(\frac{q(x)}{p(x)}\right) - f'(1)\left(\frac{q(x)}{p(x)}-1\right) \right)\dmu(x) \geq 0.
\end{equation}

Then we estimate the extended $f$-divergence using importance sampling of the integral with respect to distribution $r$, using $n$ variates $x_1,\ldots, x_n\sim_{\mathrm{iid}} p$ as:

$$
\hat{I}_{f,n}(p:q) =   \frac{1}{n} \sum_{i=1}^n f\left(\frac{q(x_i)}{p(x_i)}\right)-f'(1)\left(\frac{q(x_i)}{p(x_i)}-1\right) \geq 0.
$$
 
%
\section{Recovering statistical distances from Bregman divergences with integral-based Bregman generators}\label{sec:BDStatDiv}

Consider a strictly convex and smooth integral-based Bregman generator $F(\theta)=F_\calP(p_\theta)$ where $\calP=\{p_\theta\ :\ \theta\in\Theta\}$.
 Let $F^*(\eta)$ denote its Legendre-Fenchel convex conjugate.
The dually flat geometry induced by the dual potential functions $F(\theta)$ and $F^*(\eta)$ yields a canonical Legendre-Fenchel divergence
\begin{equation}
L_F(\theta_1:\eta_2):=F(\theta_1)+F^*(\eta_2)-\theta_1^\top\eta_2
\end{equation}
 which amounts to an equivalent Bregman divergence:
\begin{equation}
B_F(\theta_1,\theta_2):=F(\theta_1)-F(\theta_2)-(\theta_1-\theta_2)^\top \nabla F(\theta_2).
\end{equation}

These parametric Legendre-Fenchel/Bregman divergences can be interpreted as an equivalent statistical divergence  between two densities of the exponential family:
We have 
\begin{equation}
D_{F,\calP}[p_{\theta_1}:p_{\theta_2}] := L_F(\theta_1:\eta_2) = B_F(\theta_1:\theta_2).
\end{equation}

Thus by relaxing the distributions $p_{\theta_1}$ and  $p_{\theta_2}$ to {\em arbitrary} distributions $p_1$ and $p_2$, we have recovered a statistical distance $D_F[p_1:p_2]$ from an integral-based Bregman generator $F(\theta)=F_\calP(p_\theta)$.

Let us illustrate this statistical distance construction for two case studies: (1) the exponential family, and (2) the mixture family:

\begin{itemize}
\item Consider an exponential family $\calE$ of order $D$ with densities defined according to a dominating measure $\mu$:
\begin{equation}
\calE=\{p_\theta(x)=\exp(\theta^\top t(x)-F(\theta))\ :\ \theta\in\Theta\},
\end{equation}
where the natural parameter $\theta$ and the sufficient statistic vector $t(x)$ belong to $\bbR^D$.
We have the integral-based Bregman generator:
\begin{equation}
F(\theta)=F_\calE(p_\theta)=\log\left(\int\exp(\theta^\top t(x))\dmu(x)\right),
\end{equation} 
and the dual convex conjugate 
\begin{equation}
F^*(\eta)=-h(p_\theta)=\int p(x)\log p(x)\dmu(x),
\end{equation}
where $h(p)=-\int p(x)\log p(x)\dmu(x)$ denotes Shannon's entropy. 

Let $\lambda(i)$ denotes the $i$-th coordinates of vector $\lambda$.
Let us calculate the {\em inner product} $\theta_1^\top\eta_2=\sum_i \theta_1(i)\eta_2(i)$ of the Legendre-Fenchel divergence.
We have $\eta_2(i)=E_{p_{\theta_2}}[t_i(x)]$.
Using the linear property of the expectation $E[\cdot]$, 
we find that $\sum_i \theta_1(i)\eta_2(i)= E_{p_{\theta_2}}\left[\sum_i \theta_1(i) t_i(x)\right]$.  
Moreover, we have $\sum_i \theta_1(i) t_i(x)=\left(\log p_{\theta_1}(x)\right)+F(\theta_1)$.
It follows that  we have:
\begin{equation}
\theta_1^\top\eta_2 = E_{p_{\theta_2}}\left[\log p_{\theta_1}+F(\theta_1)\right]=F(\theta_1)+E_{p_{\theta_2}}\left[\log p_{\theta_1}\right].
\end{equation}

Thus we get
\begin{eqnarray}
D_{F,\calE}[p_{\theta_1}:p_{\theta_2}] &=& F(\theta_1)+F^*(\eta_2)-\theta_1^\top\eta_2,\\
&=& F(\theta_1)-h(p_{\theta_2})-E_{p_{\theta_2}}[\log p_{\theta_1}]-F(\theta_1),\\
&=& E_{p_{\theta_2}}\left[\log \frac{p_{\theta_2}}{p_{\theta_1}}\right]=:D_{\KL^*}[p_{\theta_1}:p_{\theta_2}],
\end{eqnarray}

By relaxing the densities $p_{\theta_1}$ and $p_{\theta_2}$ to be $p_1$ and $p_2$, we get the {\em reverse KL divergence} between $p_1$ and $p_2$ from the dually flat structure induced by the integral-based log-normalizer of an exponential family:
\begin{equation}
D_{\KL^*}[p_1:p_2]=E_{p_{2}}\left[\log \frac{p_{2}}{p_{1}}\right]=\int p_2(x)\log \frac{p_2(x)}{p_1(x)}\dmu(x).
\end{equation}

Thus, we have recovered the statistical distance $D_{\KL^*}$ from $D_{F,\calE}$.

The dual divergence $D^*[p_1:p_2]:=D[p_2:p_1]$ is obtained by swapping the distribution parameter orders.
We have:
\begin{equation}
D_{\KL^*}^*[p_1:p_2]:=D_{\KL^*}[p_2:p_1]= E_{p_{1}}\left[\log \frac{p_{1}}{p_{2}}\right]=:D_{\KL}[p_1:p_2],
\end{equation}
and $D_{\KL^*}[p_1:p_2]=D_{\KL^*}^*[p_2:p_1]=D_{\KL}[p_2:p_1]$.

To summarize, the canonical Legendre-Fenchel divergence associated with the log-normalizer of an exponential family amounts to
 the statistical reverse Kullback-Leibler divergence between $p_{\theta_1}$ and $p_{\theta_1}$ (or the KL divergence between the swapped corresponding densities): $D_\KL[p_{\theta_1}:p_{\theta_2}]=B_F(\theta_2:\theta_1)=L_F(\theta_2:\eta_1)$. 
Notice that it is easy to check that $D_\KL[p_{\theta_1}:p_{\theta_2}]=B_F(\theta_2:\theta_1)$~\cite{KLEF-2001,Bregman-2005}.
Here, we took the opposite direction by constructing $D_\KL$ from $B_F$.

We may consider an auxiliary carrier term $k(x)$ so that the densities write
$p_\theta(x)=\exp(\theta^\top t(x)-F(\theta)+k(x))$. Then the dual convex conjugate writes~\cite{crossEF-2010} 
as $F^*(\eta)=-h(p_\theta)+E_{p_\theta}[k(x)]$.

\item In this second example, we consider a mixture family 
\begin{equation}
\calM=\left\{m_\theta=\sum_{i=1}^D \theta_i p_i(x)+ (1-\sum_{i=1}^D \theta_i)p_0(x)\right\},
\end{equation}
 where $p_0,\ldots, p_D$ are $D+1$ linearly independent probability densities.
The integral-based Bregman generator $F$ is chosen as Shannon negentropy: 
\begin{equation}
F(\theta)=F_\calM(m_\theta)=-h(m_\theta).
\end{equation}

We have 
\begin{equation}
\eta_i=[\nabla F(\theta)]_i=\int (p_i(x)-p_0(x))\log m_\theta(x)\dmu(x),
\end{equation} 
and the dual convex potential function 
is 
\begin{equation}
F^*(\eta)=-\int p_0(x)\log m_\theta(x)\dmu(x)=h^\times(p_0:m_\theta),
\end{equation}
i.e., the cross-entropy between the density $p_0$ and the mixture $m_\theta$.
Let us calculate the inner product $\theta_1^\top\eta_2$ of the Legendre-Fenchel divergence  as follows:

\begin{eqnarray}
\sum_i \theta_1(i)\int (p_i(x)-p_0(x))\log m_{\theta_2}(x)\dmu(x) &=& 
\int  \sum_i \theta_1(i) p_i(x)\log m_{\theta_2}(x)\dmu(x)  \nonumber \\ 
&& -\sum_i \theta_1(i)  p_0(x)\log m_{\theta_2}(x)\dmu(x).
\end{eqnarray}

That is
$$
\theta_1^\top \eta_2=\int  \sum_i \theta_1(i) p_i\log m_{\theta_2}\dmu-\sum_i \theta_1(i)  p_0\log m_{\theta_2}\dmu
$$

Thus it follows that we have the following statistical distance:
\begin{eqnarray}
D_{F,\calM}[m_{\theta_1}:m_{\theta_2}] &:=& F(\theta_1)+F^*(\eta_2)-\theta_1^\top\eta_2,\\
&=& -h(m_{\theta_1})-\int p_0(x)\log m_{\theta_2}(x)\dmu(x) - 
\int  \sum_i \theta_1(i) p_i(x)\log m_{\theta_2}(x)\dmu(x) \nonumber \\ 
&&+\sum_i \theta_1(i)  p_0(x)\log m_{\theta_2}(x)\dmu(x),\\
&=&-h(m_{\theta_1})-\int ((1-\sum_i \theta_1(i))p_0(x)+\sum_i \theta_1(i)p_i(x))\log m_{\theta_2}(x)\dmu(x),\\
&=& -h(m_{\theta_1})-\int m_{\theta_1}(x)\log m_{\theta_2}(x)\dmu(x),\\
&=& \int  m_{\theta_1}(x)\log \frac{m_{\theta_1}(x)}{m_{\theta_2}(x)} \dmu(x),\\
&=& D_\KL[m_{\theta_1}:m_{\theta_2}].
\end{eqnarray}

Thus we have $D_\KL[m_{\theta_1}:m_{\theta_2}]=B_F(\theta_1:\theta_2)$.
By relaxing the mixture densities $m_{\theta_1}$ and $m_{\theta_2}$ to arbitrary densities $m_1$
 and $m_2$, we find that the dually flat geometry induced by the negentropy of densities of a mixture family induces a 
statistical distance which corresponds to the KL divergence.
That is, we have recovered the statistical distance $D_{\KL}$ from $D_{F,\calM}$.
\end{itemize}

Let us define the full dual parameter Legendre-Fenchel divergence as:
\begin{equation}\label{eq:LFsecond}
L_F(\theta_1,\eta_1:\theta_2,\eta_2) := F(\theta_1) + F^*(\eta_2) - \nabla F^*(\eta_1)^\top \nabla F(\theta_2).
\end{equation}

The purpose of introducing this expression is as follows:
When $F$ is an integral-based generator, the expression of Eq.~\ref{eq:LFsecond} can be expressed as an integral form.
Notice that it requires to explicit the gradient $\nabla F^*$ of the convex conjugate using an integral form.

\end{document}

%% file: Fig/Gconvex12.latex
\setlength{\unitlength}{0.240900pt}
\ifx\plotpoint\undefined\newsavebox{\plotpoint}\fi
\sbox{\plotpoint}{\rule[-0.200pt]{0.400pt}{0.400pt}}%
\begin{picture}(1500,900)(0,0)
\sbox{\plotpoint}{\rule[-0.200pt]{0.400pt}{0.400pt}}%
\put(130.0,82.0){\rule[-0.200pt]{4.818pt}{0.400pt}}
\put(110,82){\makebox(0,0)[r]{$-2.2$}}
\put(1419.0,82.0){\rule[-0.200pt]{4.818pt}{0.400pt}}
\put(130.0,169.0){\rule[-0.200pt]{4.818pt}{0.400pt}}
\put(110,169){\makebox(0,0)[r]{$-2.1$}}
\put(1419.0,169.0){\rule[-0.200pt]{4.818pt}{0.400pt}}
\put(130.0,256.0){\rule[-0.200pt]{4.818pt}{0.400pt}}
\put(110,256){\makebox(0,0)[r]{$-2$}}
\put(1419.0,256.0){\rule[-0.200pt]{4.818pt}{0.400pt}}
\put(130.0,342.0){\rule[-0.200pt]{4.818pt}{0.400pt}}
\put(110,342){\makebox(0,0)[r]{$-1.9$}}
\put(1419.0,342.0){\rule[-0.200pt]{4.818pt}{0.400pt}}
\put(130.0,429.0){\rule[-0.200pt]{4.818pt}{0.400pt}}
\put(110,429){\makebox(0,0)[r]{$-1.8$}}
\put(1419.0,429.0){\rule[-0.200pt]{4.818pt}{0.400pt}}
\put(130.0,516.0){\rule[-0.200pt]{4.818pt}{0.400pt}}
\put(110,516){\makebox(0,0)[r]{$-1.7$}}
\put(1419.0,516.0){\rule[-0.200pt]{4.818pt}{0.400pt}}
\put(130.0,603.0){\rule[-0.200pt]{4.818pt}{0.400pt}}
\put(110,603){\makebox(0,0)[r]{$-1.6$}}
\put(1419.0,603.0){\rule[-0.200pt]{4.818pt}{0.400pt}}
\put(130.0,689.0){\rule[-0.200pt]{4.818pt}{0.400pt}}
\put(110,689){\makebox(0,0)[r]{$-1.5$}}
\put(1419.0,689.0){\rule[-0.200pt]{4.818pt}{0.400pt}}
\put(130.0,776.0){\rule[-0.200pt]{4.818pt}{0.400pt}}
\put(110,776){\makebox(0,0)[r]{$-1.4$}}
\put(1419.0,776.0){\rule[-0.200pt]{4.818pt}{0.400pt}}
\put(130.0,82.0){\rule[-0.200pt]{0.400pt}{4.818pt}}
\put(130,41){\makebox(0,0){$0$}}
\put(130.0,756.0){\rule[-0.200pt]{0.400pt}{4.818pt}}
\put(261.0,82.0){\rule[-0.200pt]{0.400pt}{4.818pt}}
\put(261,41){\makebox(0,0){$0.1$}}
\put(261.0,756.0){\rule[-0.200pt]{0.400pt}{4.818pt}}
\put(392.0,82.0){\rule[-0.200pt]{0.400pt}{4.818pt}}
\put(392,41){\makebox(0,0){$0.2$}}
\put(392.0,756.0){\rule[-0.200pt]{0.400pt}{4.818pt}}
\put(523.0,82.0){\rule[-0.200pt]{0.400pt}{4.818pt}}
\put(523,41){\makebox(0,0){$0.3$}}
\put(523.0,756.0){\rule[-0.200pt]{0.400pt}{4.818pt}}
\put(654.0,82.0){\rule[-0.200pt]{0.400pt}{4.818pt}}
\put(654,41){\makebox(0,0){$0.4$}}
\put(654.0,756.0){\rule[-0.200pt]{0.400pt}{4.818pt}}
\put(785.0,82.0){\rule[-0.200pt]{0.400pt}{4.818pt}}
\put(785,41){\makebox(0,0){$0.5$}}
\put(785.0,756.0){\rule[-0.200pt]{0.400pt}{4.818pt}}
\put(915.0,82.0){\rule[-0.200pt]{0.400pt}{4.818pt}}
\put(915,41){\makebox(0,0){$0.6$}}
\put(915.0,756.0){\rule[-0.200pt]{0.400pt}{4.818pt}}
\put(1046.0,82.0){\rule[-0.200pt]{0.400pt}{4.818pt}}
\put(1046,41){\makebox(0,0){$0.7$}}
\put(1046.0,756.0){\rule[-0.200pt]{0.400pt}{4.818pt}}
\put(1177.0,82.0){\rule[-0.200pt]{0.400pt}{4.818pt}}
\put(1177,41){\makebox(0,0){$0.8$}}
\put(1177.0,756.0){\rule[-0.200pt]{0.400pt}{4.818pt}}
\put(1308.0,82.0){\rule[-0.200pt]{0.400pt}{4.818pt}}
\put(1308,41){\makebox(0,0){$0.9$}}
\put(1308.0,756.0){\rule[-0.200pt]{0.400pt}{4.818pt}}
\put(1439.0,82.0){\rule[-0.200pt]{0.400pt}{4.818pt}}
\put(1439,41){\makebox(0,0){$1$}}
\put(1439.0,756.0){\rule[-0.200pt]{0.400pt}{4.818pt}}
\put(130.0,82.0){\rule[-0.200pt]{0.400pt}{167.185pt}}
\put(130.0,82.0){\rule[-0.200pt]{315.338pt}{0.400pt}}
\put(1439.0,82.0){\rule[-0.200pt]{0.400pt}{167.185pt}}
\put(130.0,776.0){\rule[-0.200pt]{315.338pt}{0.400pt}}
\put(784,838){\makebox(0,0){Negative entropy of a $w$-mixture}}
\put(1279,735){\makebox(0,0)[r]{$m(x;\eta)=(1-\eta)p(x;\mu_0=0,\sigma_0=1)+\eta p(x;\mu_1=3,\sigma_1=1)$}}
\put(1299.0,735.0){\rule[-0.200pt]{24.090pt}{0.400pt}}
\put(143,729){\usebox{\plotpoint}}
\multiput(143.58,725.39)(0.493,-0.972){23}{\rule{0.119pt}{0.869pt}}
\multiput(142.17,727.20)(13.000,-23.196){2}{\rule{0.400pt}{0.435pt}}
\multiput(156.58,700.39)(0.493,-0.972){23}{\rule{0.119pt}{0.869pt}}
\multiput(155.17,702.20)(13.000,-23.196){2}{\rule{0.400pt}{0.435pt}}
\multiput(169.58,676.03)(0.493,-0.774){23}{\rule{0.119pt}{0.715pt}}
\multiput(168.17,677.52)(13.000,-18.515){2}{\rule{0.400pt}{0.358pt}}
\multiput(182.58,656.16)(0.493,-0.734){23}{\rule{0.119pt}{0.685pt}}
\multiput(181.17,657.58)(13.000,-17.579){2}{\rule{0.400pt}{0.342pt}}
\multiput(195.58,637.33)(0.494,-0.680){25}{\rule{0.119pt}{0.643pt}}
\multiput(194.17,638.67)(14.000,-17.666){2}{\rule{0.400pt}{0.321pt}}
\multiput(209.58,618.29)(0.493,-0.695){23}{\rule{0.119pt}{0.654pt}}
\multiput(208.17,619.64)(13.000,-16.643){2}{\rule{0.400pt}{0.327pt}}
\multiput(222.58,600.29)(0.493,-0.695){23}{\rule{0.119pt}{0.654pt}}
\multiput(221.17,601.64)(13.000,-16.643){2}{\rule{0.400pt}{0.327pt}}
\multiput(235.58,582.67)(0.493,-0.576){23}{\rule{0.119pt}{0.562pt}}
\multiput(234.17,583.83)(13.000,-13.834){2}{\rule{0.400pt}{0.281pt}}
\multiput(248.58,567.67)(0.493,-0.576){23}{\rule{0.119pt}{0.562pt}}
\multiput(247.17,568.83)(13.000,-13.834){2}{\rule{0.400pt}{0.281pt}}
\multiput(261.58,552.80)(0.493,-0.536){23}{\rule{0.119pt}{0.531pt}}
\multiput(260.17,553.90)(13.000,-12.898){2}{\rule{0.400pt}{0.265pt}}
\multiput(274.58,538.80)(0.493,-0.536){23}{\rule{0.119pt}{0.531pt}}
\multiput(273.17,539.90)(13.000,-12.898){2}{\rule{0.400pt}{0.265pt}}
\multiput(287.58,524.67)(0.493,-0.576){23}{\rule{0.119pt}{0.562pt}}
\multiput(286.17,525.83)(13.000,-13.834){2}{\rule{0.400pt}{0.281pt}}
\multiput(300.00,510.92)(0.652,-0.491){17}{\rule{0.620pt}{0.118pt}}
\multiput(300.00,511.17)(11.713,-10.000){2}{\rule{0.310pt}{0.400pt}}
\multiput(313.00,500.92)(0.539,-0.492){21}{\rule{0.533pt}{0.119pt}}
\multiput(313.00,501.17)(11.893,-12.000){2}{\rule{0.267pt}{0.400pt}}
\multiput(326.00,488.92)(0.539,-0.492){21}{\rule{0.533pt}{0.119pt}}
\multiput(326.00,489.17)(11.893,-12.000){2}{\rule{0.267pt}{0.400pt}}
\multiput(339.00,476.92)(0.582,-0.492){21}{\rule{0.567pt}{0.119pt}}
\multiput(339.00,477.17)(12.824,-12.000){2}{\rule{0.283pt}{0.400pt}}
\multiput(353.00,464.92)(0.590,-0.492){19}{\rule{0.573pt}{0.118pt}}
\multiput(353.00,465.17)(11.811,-11.000){2}{\rule{0.286pt}{0.400pt}}
\multiput(366.00,453.93)(0.728,-0.489){15}{\rule{0.678pt}{0.118pt}}
\multiput(366.00,454.17)(11.593,-9.000){2}{\rule{0.339pt}{0.400pt}}
\multiput(379.00,444.92)(0.590,-0.492){19}{\rule{0.573pt}{0.118pt}}
\multiput(379.00,445.17)(11.811,-11.000){2}{\rule{0.286pt}{0.400pt}}
\multiput(392.00,433.93)(0.824,-0.488){13}{\rule{0.750pt}{0.117pt}}
\multiput(392.00,434.17)(11.443,-8.000){2}{\rule{0.375pt}{0.400pt}}
\multiput(405.00,425.93)(0.728,-0.489){15}{\rule{0.678pt}{0.118pt}}
\multiput(405.00,426.17)(11.593,-9.000){2}{\rule{0.339pt}{0.400pt}}
\multiput(418.00,416.92)(0.652,-0.491){17}{\rule{0.620pt}{0.118pt}}
\multiput(418.00,417.17)(11.713,-10.000){2}{\rule{0.310pt}{0.400pt}}
\multiput(431.00,406.93)(0.950,-0.485){11}{\rule{0.843pt}{0.117pt}}
\multiput(431.00,407.17)(11.251,-7.000){2}{\rule{0.421pt}{0.400pt}}
\multiput(444.00,399.93)(0.824,-0.488){13}{\rule{0.750pt}{0.117pt}}
\multiput(444.00,400.17)(11.443,-8.000){2}{\rule{0.375pt}{0.400pt}}
\multiput(457.00,391.93)(0.824,-0.488){13}{\rule{0.750pt}{0.117pt}}
\multiput(457.00,392.17)(11.443,-8.000){2}{\rule{0.375pt}{0.400pt}}
\multiput(470.00,383.93)(1.378,-0.477){7}{\rule{1.140pt}{0.115pt}}
\multiput(470.00,384.17)(10.634,-5.000){2}{\rule{0.570pt}{0.400pt}}
\multiput(483.00,378.93)(0.890,-0.488){13}{\rule{0.800pt}{0.117pt}}
\multiput(483.00,379.17)(12.340,-8.000){2}{\rule{0.400pt}{0.400pt}}
\multiput(497.00,370.93)(0.950,-0.485){11}{\rule{0.843pt}{0.117pt}}
\multiput(497.00,371.17)(11.251,-7.000){2}{\rule{0.421pt}{0.400pt}}
\multiput(510.00,363.93)(1.378,-0.477){7}{\rule{1.140pt}{0.115pt}}
\multiput(510.00,364.17)(10.634,-5.000){2}{\rule{0.570pt}{0.400pt}}
\multiput(523.00,358.93)(1.123,-0.482){9}{\rule{0.967pt}{0.116pt}}
\multiput(523.00,359.17)(10.994,-6.000){2}{\rule{0.483pt}{0.400pt}}
\multiput(536.00,352.93)(1.378,-0.477){7}{\rule{1.140pt}{0.115pt}}
\multiput(536.00,353.17)(10.634,-5.000){2}{\rule{0.570pt}{0.400pt}}
\multiput(549.00,347.93)(1.123,-0.482){9}{\rule{0.967pt}{0.116pt}}
\multiput(549.00,348.17)(10.994,-6.000){2}{\rule{0.483pt}{0.400pt}}
\multiput(562.00,341.93)(1.378,-0.477){7}{\rule{1.140pt}{0.115pt}}
\multiput(562.00,342.17)(10.634,-5.000){2}{\rule{0.570pt}{0.400pt}}
\multiput(575.00,336.94)(1.797,-0.468){5}{\rule{1.400pt}{0.113pt}}
\multiput(575.00,337.17)(10.094,-4.000){2}{\rule{0.700pt}{0.400pt}}
\multiput(588.00,332.95)(2.695,-0.447){3}{\rule{1.833pt}{0.108pt}}
\multiput(588.00,333.17)(9.195,-3.000){2}{\rule{0.917pt}{0.400pt}}
\multiput(601.00,329.94)(1.797,-0.468){5}{\rule{1.400pt}{0.113pt}}
\multiput(601.00,330.17)(10.094,-4.000){2}{\rule{0.700pt}{0.400pt}}
\multiput(614.00,325.94)(1.797,-0.468){5}{\rule{1.400pt}{0.113pt}}
\multiput(614.00,326.17)(10.094,-4.000){2}{\rule{0.700pt}{0.400pt}}
\put(627,321.17){\rule{2.900pt}{0.400pt}}
\multiput(627.00,322.17)(7.981,-2.000){2}{\rule{1.450pt}{0.400pt}}
\multiput(641.00,319.94)(1.797,-0.468){5}{\rule{1.400pt}{0.113pt}}
\multiput(641.00,320.17)(10.094,-4.000){2}{\rule{0.700pt}{0.400pt}}
\multiput(654.00,315.94)(1.797,-0.468){5}{\rule{1.400pt}{0.113pt}}
\multiput(654.00,316.17)(10.094,-4.000){2}{\rule{0.700pt}{0.400pt}}
\put(667,311.67){\rule{3.132pt}{0.400pt}}
\multiput(667.00,312.17)(6.500,-1.000){2}{\rule{1.566pt}{0.400pt}}
\put(680,310.17){\rule{2.700pt}{0.400pt}}
\multiput(680.00,311.17)(7.396,-2.000){2}{\rule{1.350pt}{0.400pt}}
\put(693,308.67){\rule{3.132pt}{0.400pt}}
\multiput(693.00,309.17)(6.500,-1.000){2}{\rule{1.566pt}{0.400pt}}
\multiput(706.00,307.95)(2.695,-0.447){3}{\rule{1.833pt}{0.108pt}}
\multiput(706.00,308.17)(9.195,-3.000){2}{\rule{0.917pt}{0.400pt}}
\put(719,304.67){\rule{3.132pt}{0.400pt}}
\multiput(719.00,305.17)(6.500,-1.000){2}{\rule{1.566pt}{0.400pt}}
\put(745,303.67){\rule{3.132pt}{0.400pt}}
\multiput(745.00,304.17)(6.500,-1.000){2}{\rule{1.566pt}{0.400pt}}
\put(758,302.17){\rule{2.700pt}{0.400pt}}
\multiput(758.00,303.17)(7.396,-2.000){2}{\rule{1.350pt}{0.400pt}}
\put(732.0,305.0){\rule[-0.200pt]{3.132pt}{0.400pt}}
\put(785,301.67){\rule{3.132pt}{0.400pt}}
\multiput(785.00,301.17)(6.500,1.000){2}{\rule{1.566pt}{0.400pt}}
\put(798,301.67){\rule{3.132pt}{0.400pt}}
\multiput(798.00,302.17)(6.500,-1.000){2}{\rule{1.566pt}{0.400pt}}
\put(811,302.17){\rule{2.700pt}{0.400pt}}
\multiput(811.00,301.17)(7.396,2.000){2}{\rule{1.350pt}{0.400pt}}
\put(771.0,302.0){\rule[-0.200pt]{3.373pt}{0.400pt}}
\put(837,304.17){\rule{2.700pt}{0.400pt}}
\multiput(837.00,303.17)(7.396,2.000){2}{\rule{1.350pt}{0.400pt}}
\put(850,306.17){\rule{2.700pt}{0.400pt}}
\multiput(850.00,305.17)(7.396,2.000){2}{\rule{1.350pt}{0.400pt}}
\put(863,308.17){\rule{2.700pt}{0.400pt}}
\multiput(863.00,307.17)(7.396,2.000){2}{\rule{1.350pt}{0.400pt}}
\put(876,310.17){\rule{2.700pt}{0.400pt}}
\multiput(876.00,309.17)(7.396,2.000){2}{\rule{1.350pt}{0.400pt}}
\put(889,312.17){\rule{2.700pt}{0.400pt}}
\multiput(889.00,311.17)(7.396,2.000){2}{\rule{1.350pt}{0.400pt}}
\multiput(902.00,314.61)(2.695,0.447){3}{\rule{1.833pt}{0.108pt}}
\multiput(902.00,313.17)(9.195,3.000){2}{\rule{0.917pt}{0.400pt}}
\put(915,317.17){\rule{2.700pt}{0.400pt}}
\multiput(915.00,316.17)(7.396,2.000){2}{\rule{1.350pt}{0.400pt}}
\multiput(928.00,319.60)(1.943,0.468){5}{\rule{1.500pt}{0.113pt}}
\multiput(928.00,318.17)(10.887,4.000){2}{\rule{0.750pt}{0.400pt}}
\multiput(942.00,323.60)(1.797,0.468){5}{\rule{1.400pt}{0.113pt}}
\multiput(942.00,322.17)(10.094,4.000){2}{\rule{0.700pt}{0.400pt}}
\multiput(955.00,327.60)(1.797,0.468){5}{\rule{1.400pt}{0.113pt}}
\multiput(955.00,326.17)(10.094,4.000){2}{\rule{0.700pt}{0.400pt}}
\multiput(968.00,331.61)(2.695,0.447){3}{\rule{1.833pt}{0.108pt}}
\multiput(968.00,330.17)(9.195,3.000){2}{\rule{0.917pt}{0.400pt}}
\multiput(981.00,334.59)(1.378,0.477){7}{\rule{1.140pt}{0.115pt}}
\multiput(981.00,333.17)(10.634,5.000){2}{\rule{0.570pt}{0.400pt}}
\multiput(994.00,339.60)(1.797,0.468){5}{\rule{1.400pt}{0.113pt}}
\multiput(994.00,338.17)(10.094,4.000){2}{\rule{0.700pt}{0.400pt}}
\multiput(1007.00,343.59)(1.123,0.482){9}{\rule{0.967pt}{0.116pt}}
\multiput(1007.00,342.17)(10.994,6.000){2}{\rule{0.483pt}{0.400pt}}
\multiput(1020.00,349.60)(1.797,0.468){5}{\rule{1.400pt}{0.113pt}}
\multiput(1020.00,348.17)(10.094,4.000){2}{\rule{0.700pt}{0.400pt}}
\multiput(1033.00,353.59)(0.950,0.485){11}{\rule{0.843pt}{0.117pt}}
\multiput(1033.00,352.17)(11.251,7.000){2}{\rule{0.421pt}{0.400pt}}
\multiput(1046.00,360.59)(1.378,0.477){7}{\rule{1.140pt}{0.115pt}}
\multiput(1046.00,359.17)(10.634,5.000){2}{\rule{0.570pt}{0.400pt}}
\multiput(1059.00,365.59)(1.123,0.482){9}{\rule{0.967pt}{0.116pt}}
\multiput(1059.00,364.17)(10.994,6.000){2}{\rule{0.483pt}{0.400pt}}
\multiput(1072.00,371.59)(0.890,0.488){13}{\rule{0.800pt}{0.117pt}}
\multiput(1072.00,370.17)(12.340,8.000){2}{\rule{0.400pt}{0.400pt}}
\multiput(1086.00,379.59)(0.950,0.485){11}{\rule{0.843pt}{0.117pt}}
\multiput(1086.00,378.17)(11.251,7.000){2}{\rule{0.421pt}{0.400pt}}
\multiput(1099.00,386.59)(0.824,0.488){13}{\rule{0.750pt}{0.117pt}}
\multiput(1099.00,385.17)(11.443,8.000){2}{\rule{0.375pt}{0.400pt}}
\multiput(1112.00,394.59)(0.950,0.485){11}{\rule{0.843pt}{0.117pt}}
\multiput(1112.00,393.17)(11.251,7.000){2}{\rule{0.421pt}{0.400pt}}
\multiput(1125.00,401.59)(0.824,0.488){13}{\rule{0.750pt}{0.117pt}}
\multiput(1125.00,400.17)(11.443,8.000){2}{\rule{0.375pt}{0.400pt}}
\multiput(1138.00,409.59)(0.728,0.489){15}{\rule{0.678pt}{0.118pt}}
\multiput(1138.00,408.17)(11.593,9.000){2}{\rule{0.339pt}{0.400pt}}
\multiput(1151.00,418.59)(0.824,0.488){13}{\rule{0.750pt}{0.117pt}}
\multiput(1151.00,417.17)(11.443,8.000){2}{\rule{0.375pt}{0.400pt}}
\multiput(1164.00,426.59)(0.728,0.489){15}{\rule{0.678pt}{0.118pt}}
\multiput(1164.00,425.17)(11.593,9.000){2}{\rule{0.339pt}{0.400pt}}
\multiput(1177.00,435.58)(0.590,0.492){19}{\rule{0.573pt}{0.118pt}}
\multiput(1177.00,434.17)(11.811,11.000){2}{\rule{0.286pt}{0.400pt}}
\multiput(1190.00,446.59)(0.728,0.489){15}{\rule{0.678pt}{0.118pt}}
\multiput(1190.00,445.17)(11.593,9.000){2}{\rule{0.339pt}{0.400pt}}
\multiput(1203.00,455.58)(0.590,0.492){19}{\rule{0.573pt}{0.118pt}}
\multiput(1203.00,454.17)(11.811,11.000){2}{\rule{0.286pt}{0.400pt}}
\multiput(1216.00,466.58)(0.637,0.492){19}{\rule{0.609pt}{0.118pt}}
\multiput(1216.00,465.17)(12.736,11.000){2}{\rule{0.305pt}{0.400pt}}
\multiput(1230.00,477.58)(0.497,0.493){23}{\rule{0.500pt}{0.119pt}}
\multiput(1230.00,476.17)(11.962,13.000){2}{\rule{0.250pt}{0.400pt}}
\multiput(1243.00,490.58)(0.539,0.492){21}{\rule{0.533pt}{0.119pt}}
\multiput(1243.00,489.17)(11.893,12.000){2}{\rule{0.267pt}{0.400pt}}
\multiput(1256.00,502.58)(0.652,0.491){17}{\rule{0.620pt}{0.118pt}}
\multiput(1256.00,501.17)(11.713,10.000){2}{\rule{0.310pt}{0.400pt}}
\multiput(1269.58,512.00)(0.493,0.576){23}{\rule{0.119pt}{0.562pt}}
\multiput(1268.17,512.00)(13.000,13.834){2}{\rule{0.400pt}{0.281pt}}
\multiput(1282.58,527.00)(0.493,0.536){23}{\rule{0.119pt}{0.531pt}}
\multiput(1281.17,527.00)(13.000,12.898){2}{\rule{0.400pt}{0.265pt}}
\multiput(1295.58,541.00)(0.493,0.536){23}{\rule{0.119pt}{0.531pt}}
\multiput(1294.17,541.00)(13.000,12.898){2}{\rule{0.400pt}{0.265pt}}
\multiput(1308.58,555.00)(0.493,0.616){23}{\rule{0.119pt}{0.592pt}}
\multiput(1307.17,555.00)(13.000,14.771){2}{\rule{0.400pt}{0.296pt}}
\multiput(1321.58,571.00)(0.493,0.576){23}{\rule{0.119pt}{0.562pt}}
\multiput(1320.17,571.00)(13.000,13.834){2}{\rule{0.400pt}{0.281pt}}
\multiput(1334.58,586.00)(0.493,0.616){23}{\rule{0.119pt}{0.592pt}}
\multiput(1333.17,586.00)(13.000,14.771){2}{\rule{0.400pt}{0.296pt}}
\multiput(1347.58,602.00)(0.493,0.695){23}{\rule{0.119pt}{0.654pt}}
\multiput(1346.17,602.00)(13.000,16.643){2}{\rule{0.400pt}{0.327pt}}
\multiput(1360.58,620.00)(0.494,0.717){25}{\rule{0.119pt}{0.671pt}}
\multiput(1359.17,620.00)(14.000,18.606){2}{\rule{0.400pt}{0.336pt}}
\multiput(1374.58,640.00)(0.493,0.734){23}{\rule{0.119pt}{0.685pt}}
\multiput(1373.17,640.00)(13.000,17.579){2}{\rule{0.400pt}{0.342pt}}
\multiput(1387.58,659.00)(0.493,0.774){23}{\rule{0.119pt}{0.715pt}}
\multiput(1386.17,659.00)(13.000,18.515){2}{\rule{0.400pt}{0.358pt}}
\multiput(1400.58,679.00)(0.493,0.933){23}{\rule{0.119pt}{0.838pt}}
\multiput(1399.17,679.00)(13.000,22.260){2}{\rule{0.400pt}{0.419pt}}
\multiput(1413.58,703.00)(0.493,1.052){23}{\rule{0.119pt}{0.931pt}}
\multiput(1412.17,703.00)(13.000,25.068){2}{\rule{0.400pt}{0.465pt}}
\put(824.0,304.0){\rule[-0.200pt]{3.132pt}{0.400pt}}
\put(1279,694){\makebox(0,0)[r]{$m(x;\eta)=(1-\eta)p(x;\mu_0=0,\sigma_0=1)+\eta p(x;\mu_1=3,\sigma_1=2)$}}
\multiput(1299,694)(20.756,0.000){5}{\usebox{\plotpoint}}
\put(1399,694){\usebox{\plotpoint}}
\put(143,153){\usebox{\plotpoint}}
\put(143.00,153.00){\usebox{\plotpoint}}
\put(161.44,143.49){\usebox{\plotpoint}}
\put(180.88,136.35){\usebox{\plotpoint}}
\put(200.06,128.55){\usebox{\plotpoint}}
\put(219.70,121.89){\usebox{\plotpoint}}
\put(239.57,115.95){\usebox{\plotpoint}}
\put(259.56,110.44){\usebox{\plotpoint}}
\put(279.83,106.21){\usebox{\plotpoint}}
\put(299.91,101.02){\usebox{\plotpoint}}
\put(320.43,97.86){\usebox{\plotpoint}}
\put(340.95,94.72){\usebox{\plotpoint}}
\put(361.55,92.34){\usebox{\plotpoint}}
\put(381.95,89.00){\usebox{\plotpoint}}
\put(402.71,89.00){\usebox{\plotpoint}}
\put(423.43,88.00){\usebox{\plotpoint}}
\put(444.14,87.01){\usebox{\plotpoint}}
\put(464.86,88.00){\usebox{\plotpoint}}
\put(485.57,89.18){\usebox{\plotpoint}}
\put(506.30,90.00){\usebox{\plotpoint}}
\put(527.00,91.31){\usebox{\plotpoint}}
\put(547.60,93.78){\usebox{\plotpoint}}
\put(568.16,96.47){\usebox{\plotpoint}}
\put(588.55,100.04){\usebox{\plotpoint}}
\put(609.06,102.86){\usebox{\plotpoint}}
\put(629.29,107.49){\usebox{\plotpoint}}
\put(649.40,112.58){\usebox{\plotpoint}}
\put(669.48,117.76){\usebox{\plotpoint}}
\put(689.64,122.48){\usebox{\plotpoint}}
\put(709.74,127.44){\usebox{\plotpoint}}
\put(729.36,134.19){\usebox{\plotpoint}}
\put(748.88,141.20){\usebox{\plotpoint}}
\put(768.19,148.70){\usebox{\plotpoint}}
\put(787.67,155.82){\usebox{\plotpoint}}
\put(807.28,162.57){\usebox{\plotpoint}}
\put(825.82,171.84){\usebox{\plotpoint}}
\put(844.67,180.54){\usebox{\plotpoint}}
\put(863.84,188.45){\usebox{\plotpoint}}
\put(882.11,198.29){\usebox{\plotpoint}}
\put(900.74,207.42){\usebox{\plotpoint}}
\put(918.73,217.72){\usebox{\plotpoint}}
\put(937.16,227.23){\usebox{\plotpoint}}
\put(954.48,238.64){\usebox{\plotpoint}}
\put(972.14,249.54){\usebox{\plotpoint}}
\put(989.81,260.42){\usebox{\plotpoint}}
\put(1007.85,270.65){\usebox{\plotpoint}}
\put(1024.46,283.09){\usebox{\plotpoint}}
\put(1041.83,294.44){\usebox{\plotpoint}}
\put(1058.11,307.25){\usebox{\plotpoint}}
\put(1074.97,319.33){\usebox{\plotpoint}}
\put(1091.33,332.10){\usebox{\plotpoint}}
\put(1107.46,345.16){\usebox{\plotpoint}}
\put(1123.30,358.57){\usebox{\plotpoint}}
\put(1139.11,372.02){\usebox{\plotpoint}}
\put(1154.62,385.79){\usebox{\plotpoint}}
\put(1170.31,399.31){\usebox{\plotpoint}}
\put(1185.30,413.66){\usebox{\plotpoint}}
\put(1200.55,427.74){\usebox{\plotpoint}}
\put(1214.41,443.17){\usebox{\plotpoint}}
\put(1229.43,457.47){\usebox{\plotpoint}}
\put(1243.62,472.62){\usebox{\plotpoint}}
\put(1258.21,487.38){\usebox{\plotpoint}}
\put(1271.97,502.89){\usebox{\plotpoint}}
\put(1284.89,519.11){\usebox{\plotpoint}}
\put(1298.58,534.69){\usebox{\plotpoint}}
\put(1311.08,551.26){\usebox{\plotpoint}}
\put(1323.49,567.88){\usebox{\plotpoint}}
\put(1336.66,583.89){\usebox{\plotpoint}}
\multiput(1347,599)(11.312,17.402){2}{\usebox{\plotpoint}}
\put(1371.93,635.18){\usebox{\plotpoint}}
\put(1383.41,652.47){\usebox{\plotpoint}}
\put(1394.20,670.19){\usebox{\plotpoint}}
\put(1404.61,688.15){\usebox{\plotpoint}}
\multiput(1413,703)(9.004,18.701){2}{\usebox{\plotpoint}}
\put(1426,730){\usebox{\plotpoint}}
\put(130.0,82.0){\rule[-0.200pt]{0.400pt}{167.185pt}}
\put(130.0,82.0){\rule[-0.200pt]{315.338pt}{0.400pt}}
\put(1439.0,82.0){\rule[-0.200pt]{0.400pt}{167.185pt}}
\put(130.0,776.0){\rule[-0.200pt]{315.338pt}{0.400pt}}
\end{picture}

%% file: Fig/Fconvex12.latex
\setlength{\unitlength}{0.240900pt}
\ifx\plotpoint\undefined\newsavebox{\plotpoint}\fi
\sbox{\plotpoint}{\rule[-0.200pt]{0.400pt}{0.400pt}}%
\begin{picture}(1500,900)(0,0)
\sbox{\plotpoint}{\rule[-0.200pt]{0.400pt}{0.400pt}}%
\put(110.0,82.0){\rule[-0.200pt]{4.818pt}{0.400pt}}
\put(90,82){\makebox(0,0)[r]{$1$}}
\put(1419.0,82.0){\rule[-0.200pt]{4.818pt}{0.400pt}}
\put(110.0,169.0){\rule[-0.200pt]{4.818pt}{0.400pt}}
\put(90,169){\makebox(0,0)[r]{$1.5$}}
\put(1419.0,169.0){\rule[-0.200pt]{4.818pt}{0.400pt}}
\put(110.0,256.0){\rule[-0.200pt]{4.818pt}{0.400pt}}
\put(90,256){\makebox(0,0)[r]{$2$}}
\put(1419.0,256.0){\rule[-0.200pt]{4.818pt}{0.400pt}}
\put(110.0,342.0){\rule[-0.200pt]{4.818pt}{0.400pt}}
\put(90,342){\makebox(0,0)[r]{$2.5$}}
\put(1419.0,342.0){\rule[-0.200pt]{4.818pt}{0.400pt}}
\put(110.0,429.0){\rule[-0.200pt]{4.818pt}{0.400pt}}
\put(90,429){\makebox(0,0)[r]{$3$}}
\put(1419.0,429.0){\rule[-0.200pt]{4.818pt}{0.400pt}}
\put(110.0,516.0){\rule[-0.200pt]{4.818pt}{0.400pt}}
\put(90,516){\makebox(0,0)[r]{$3.5$}}
\put(1419.0,516.0){\rule[-0.200pt]{4.818pt}{0.400pt}}
\put(110.0,603.0){\rule[-0.200pt]{4.818pt}{0.400pt}}
\put(90,603){\makebox(0,0)[r]{$4$}}
\put(1419.0,603.0){\rule[-0.200pt]{4.818pt}{0.400pt}}
\put(110.0,689.0){\rule[-0.200pt]{4.818pt}{0.400pt}}
\put(90,689){\makebox(0,0)[r]{$4.5$}}
\put(1419.0,689.0){\rule[-0.200pt]{4.818pt}{0.400pt}}
\put(110.0,776.0){\rule[-0.200pt]{4.818pt}{0.400pt}}
\put(90,776){\makebox(0,0)[r]{$5$}}
\put(1419.0,776.0){\rule[-0.200pt]{4.818pt}{0.400pt}}
\put(110.0,82.0){\rule[-0.200pt]{0.400pt}{4.818pt}}
\put(110,41){\makebox(0,0){$0$}}
\put(110.0,756.0){\rule[-0.200pt]{0.400pt}{4.818pt}}
\put(243.0,82.0){\rule[-0.200pt]{0.400pt}{4.818pt}}
\put(243,41){\makebox(0,0){$0.1$}}
\put(243.0,756.0){\rule[-0.200pt]{0.400pt}{4.818pt}}
\put(376.0,82.0){\rule[-0.200pt]{0.400pt}{4.818pt}}
\put(376,41){\makebox(0,0){$0.2$}}
\put(376.0,756.0){\rule[-0.200pt]{0.400pt}{4.818pt}}
\put(509.0,82.0){\rule[-0.200pt]{0.400pt}{4.818pt}}
\put(509,41){\makebox(0,0){$0.3$}}
\put(509.0,756.0){\rule[-0.200pt]{0.400pt}{4.818pt}}
\put(642.0,82.0){\rule[-0.200pt]{0.400pt}{4.818pt}}
\put(642,41){\makebox(0,0){$0.4$}}
\put(642.0,756.0){\rule[-0.200pt]{0.400pt}{4.818pt}}
\put(775.0,82.0){\rule[-0.200pt]{0.400pt}{4.818pt}}
\put(775,41){\makebox(0,0){$0.5$}}
\put(775.0,756.0){\rule[-0.200pt]{0.400pt}{4.818pt}}
\put(907.0,82.0){\rule[-0.200pt]{0.400pt}{4.818pt}}
\put(907,41){\makebox(0,0){$0.6$}}
\put(907.0,756.0){\rule[-0.200pt]{0.400pt}{4.818pt}}
\put(1040.0,82.0){\rule[-0.200pt]{0.400pt}{4.818pt}}
\put(1040,41){\makebox(0,0){$0.7$}}
\put(1040.0,756.0){\rule[-0.200pt]{0.400pt}{4.818pt}}
\put(1173.0,82.0){\rule[-0.200pt]{0.400pt}{4.818pt}}
\put(1173,41){\makebox(0,0){$0.8$}}
\put(1173.0,756.0){\rule[-0.200pt]{0.400pt}{4.818pt}}
\put(1306.0,82.0){\rule[-0.200pt]{0.400pt}{4.818pt}}
\put(1306,41){\makebox(0,0){$0.9$}}
\put(1306.0,756.0){\rule[-0.200pt]{0.400pt}{4.818pt}}
\put(1439.0,82.0){\rule[-0.200pt]{0.400pt}{4.818pt}}
\put(1439,41){\makebox(0,0){$1$}}
\put(1439.0,756.0){\rule[-0.200pt]{0.400pt}{4.818pt}}
\put(110.0,82.0){\rule[-0.200pt]{0.400pt}{167.185pt}}
\put(110.0,82.0){\rule[-0.200pt]{320.156pt}{0.400pt}}
\put(1439.0,82.0){\rule[-0.200pt]{0.400pt}{167.185pt}}
\put(110.0,776.0){\rule[-0.200pt]{320.156pt}{0.400pt}}
\put(774,838){\makebox(0,0){cross-entropy $F(\theta)=h^\times(p_0(x):m(x;\eta))$ of a $w$-mixture}}
\put(1279,735){\makebox(0,0)[r]{$m(x;\eta)=(1-\eta)p(x;\mu_0=0,\sigma_0=1)+\eta p(x;\mu_1=3,\sigma_1=1)$}}
\put(1299.0,735.0){\rule[-0.200pt]{24.090pt}{0.400pt}}
\put(123,704){\usebox{\plotpoint}}
\multiput(123.58,695.99)(0.494,-2.333){25}{\rule{0.119pt}{1.929pt}}
\multiput(122.17,700.00)(14.000,-59.997){2}{\rule{0.400pt}{0.964pt}}
\multiput(137.58,634.48)(0.493,-1.567){23}{\rule{0.119pt}{1.331pt}}
\multiput(136.17,637.24)(13.000,-37.238){2}{\rule{0.400pt}{0.665pt}}
\multiput(150.58,595.63)(0.493,-1.210){23}{\rule{0.119pt}{1.054pt}}
\multiput(149.17,597.81)(13.000,-28.813){2}{\rule{0.400pt}{0.527pt}}
\multiput(163.58,565.39)(0.493,-0.972){23}{\rule{0.119pt}{0.869pt}}
\multiput(162.17,567.20)(13.000,-23.196){2}{\rule{0.400pt}{0.435pt}}
\multiput(176.58,541.09)(0.494,-0.754){25}{\rule{0.119pt}{0.700pt}}
\multiput(175.17,542.55)(14.000,-19.547){2}{\rule{0.400pt}{0.350pt}}
\multiput(190.58,520.29)(0.493,-0.695){23}{\rule{0.119pt}{0.654pt}}
\multiput(189.17,521.64)(13.000,-16.643){2}{\rule{0.400pt}{0.327pt}}
\multiput(203.58,502.54)(0.493,-0.616){23}{\rule{0.119pt}{0.592pt}}
\multiput(202.17,503.77)(13.000,-14.771){2}{\rule{0.400pt}{0.296pt}}
\multiput(216.58,486.81)(0.494,-0.534){25}{\rule{0.119pt}{0.529pt}}
\multiput(215.17,487.90)(14.000,-13.903){2}{\rule{0.400pt}{0.264pt}}
\multiput(230.00,472.92)(0.497,-0.493){23}{\rule{0.500pt}{0.119pt}}
\multiput(230.00,473.17)(11.962,-13.000){2}{\rule{0.250pt}{0.400pt}}
\multiput(243.00,459.92)(0.539,-0.492){21}{\rule{0.533pt}{0.119pt}}
\multiput(243.00,460.17)(11.893,-12.000){2}{\rule{0.267pt}{0.400pt}}
\multiput(256.00,447.92)(0.590,-0.492){19}{\rule{0.573pt}{0.118pt}}
\multiput(256.00,448.17)(11.811,-11.000){2}{\rule{0.286pt}{0.400pt}}
\multiput(269.00,436.92)(0.704,-0.491){17}{\rule{0.660pt}{0.118pt}}
\multiput(269.00,437.17)(12.630,-10.000){2}{\rule{0.330pt}{0.400pt}}
\multiput(283.00,426.92)(0.652,-0.491){17}{\rule{0.620pt}{0.118pt}}
\multiput(283.00,427.17)(11.713,-10.000){2}{\rule{0.310pt}{0.400pt}}
\multiput(296.00,416.93)(0.824,-0.488){13}{\rule{0.750pt}{0.117pt}}
\multiput(296.00,417.17)(11.443,-8.000){2}{\rule{0.375pt}{0.400pt}}
\multiput(309.00,408.93)(0.786,-0.489){15}{\rule{0.722pt}{0.118pt}}
\multiput(309.00,409.17)(12.501,-9.000){2}{\rule{0.361pt}{0.400pt}}
\multiput(323.00,399.93)(0.824,-0.488){13}{\rule{0.750pt}{0.117pt}}
\multiput(323.00,400.17)(11.443,-8.000){2}{\rule{0.375pt}{0.400pt}}
\multiput(336.00,391.93)(0.950,-0.485){11}{\rule{0.843pt}{0.117pt}}
\multiput(336.00,392.17)(11.251,-7.000){2}{\rule{0.421pt}{0.400pt}}
\multiput(349.00,384.93)(0.890,-0.488){13}{\rule{0.800pt}{0.117pt}}
\multiput(349.00,385.17)(12.340,-8.000){2}{\rule{0.400pt}{0.400pt}}
\multiput(363.00,376.93)(1.123,-0.482){9}{\rule{0.967pt}{0.116pt}}
\multiput(363.00,377.17)(10.994,-6.000){2}{\rule{0.483pt}{0.400pt}}
\multiput(376.00,370.93)(0.950,-0.485){11}{\rule{0.843pt}{0.117pt}}
\multiput(376.00,371.17)(11.251,-7.000){2}{\rule{0.421pt}{0.400pt}}
\multiput(389.00,363.93)(1.123,-0.482){9}{\rule{0.967pt}{0.116pt}}
\multiput(389.00,364.17)(10.994,-6.000){2}{\rule{0.483pt}{0.400pt}}
\multiput(402.00,357.93)(1.214,-0.482){9}{\rule{1.033pt}{0.116pt}}
\multiput(402.00,358.17)(11.855,-6.000){2}{\rule{0.517pt}{0.400pt}}
\multiput(416.00,351.93)(1.123,-0.482){9}{\rule{0.967pt}{0.116pt}}
\multiput(416.00,352.17)(10.994,-6.000){2}{\rule{0.483pt}{0.400pt}}
\multiput(429.00,345.93)(1.378,-0.477){7}{\rule{1.140pt}{0.115pt}}
\multiput(429.00,346.17)(10.634,-5.000){2}{\rule{0.570pt}{0.400pt}}
\multiput(442.00,340.93)(1.214,-0.482){9}{\rule{1.033pt}{0.116pt}}
\multiput(442.00,341.17)(11.855,-6.000){2}{\rule{0.517pt}{0.400pt}}
\multiput(456.00,334.93)(1.378,-0.477){7}{\rule{1.140pt}{0.115pt}}
\multiput(456.00,335.17)(10.634,-5.000){2}{\rule{0.570pt}{0.400pt}}
\multiput(469.00,329.93)(1.378,-0.477){7}{\rule{1.140pt}{0.115pt}}
\multiput(469.00,330.17)(10.634,-5.000){2}{\rule{0.570pt}{0.400pt}}
\multiput(482.00,324.94)(1.797,-0.468){5}{\rule{1.400pt}{0.113pt}}
\multiput(482.00,325.17)(10.094,-4.000){2}{\rule{0.700pt}{0.400pt}}
\multiput(495.00,320.93)(1.489,-0.477){7}{\rule{1.220pt}{0.115pt}}
\multiput(495.00,321.17)(11.468,-5.000){2}{\rule{0.610pt}{0.400pt}}
\multiput(509.00,315.93)(1.378,-0.477){7}{\rule{1.140pt}{0.115pt}}
\multiput(509.00,316.17)(10.634,-5.000){2}{\rule{0.570pt}{0.400pt}}
\multiput(522.00,310.94)(1.797,-0.468){5}{\rule{1.400pt}{0.113pt}}
\multiput(522.00,311.17)(10.094,-4.000){2}{\rule{0.700pt}{0.400pt}}
\multiput(535.00,306.94)(1.943,-0.468){5}{\rule{1.500pt}{0.113pt}}
\multiput(535.00,307.17)(10.887,-4.000){2}{\rule{0.750pt}{0.400pt}}
\multiput(549.00,302.94)(1.797,-0.468){5}{\rule{1.400pt}{0.113pt}}
\multiput(549.00,303.17)(10.094,-4.000){2}{\rule{0.700pt}{0.400pt}}
\multiput(562.00,298.94)(1.797,-0.468){5}{\rule{1.400pt}{0.113pt}}
\multiput(562.00,299.17)(10.094,-4.000){2}{\rule{0.700pt}{0.400pt}}
\multiput(575.00,294.94)(1.797,-0.468){5}{\rule{1.400pt}{0.113pt}}
\multiput(575.00,295.17)(10.094,-4.000){2}{\rule{0.700pt}{0.400pt}}
\multiput(588.00,290.94)(1.943,-0.468){5}{\rule{1.500pt}{0.113pt}}
\multiput(588.00,291.17)(10.887,-4.000){2}{\rule{0.750pt}{0.400pt}}
\multiput(602.00,286.94)(1.797,-0.468){5}{\rule{1.400pt}{0.113pt}}
\multiput(602.00,287.17)(10.094,-4.000){2}{\rule{0.700pt}{0.400pt}}
\multiput(615.00,282.95)(2.695,-0.447){3}{\rule{1.833pt}{0.108pt}}
\multiput(615.00,283.17)(9.195,-3.000){2}{\rule{0.917pt}{0.400pt}}
\multiput(628.00,279.94)(1.943,-0.468){5}{\rule{1.500pt}{0.113pt}}
\multiput(628.00,280.17)(10.887,-4.000){2}{\rule{0.750pt}{0.400pt}}
\multiput(642.00,275.95)(2.695,-0.447){3}{\rule{1.833pt}{0.108pt}}
\multiput(642.00,276.17)(9.195,-3.000){2}{\rule{0.917pt}{0.400pt}}
\multiput(655.00,272.94)(1.797,-0.468){5}{\rule{1.400pt}{0.113pt}}
\multiput(655.00,273.17)(10.094,-4.000){2}{\rule{0.700pt}{0.400pt}}
\multiput(668.00,268.95)(2.695,-0.447){3}{\rule{1.833pt}{0.108pt}}
\multiput(668.00,269.17)(9.195,-3.000){2}{\rule{0.917pt}{0.400pt}}
\multiput(681.00,265.95)(2.918,-0.447){3}{\rule{1.967pt}{0.108pt}}
\multiput(681.00,266.17)(9.918,-3.000){2}{\rule{0.983pt}{0.400pt}}
\multiput(695.00,262.95)(2.695,-0.447){3}{\rule{1.833pt}{0.108pt}}
\multiput(695.00,263.17)(9.195,-3.000){2}{\rule{0.917pt}{0.400pt}}
\multiput(708.00,259.95)(2.695,-0.447){3}{\rule{1.833pt}{0.108pt}}
\multiput(708.00,260.17)(9.195,-3.000){2}{\rule{0.917pt}{0.400pt}}
\multiput(721.00,256.95)(2.918,-0.447){3}{\rule{1.967pt}{0.108pt}}
\multiput(721.00,257.17)(9.918,-3.000){2}{\rule{0.983pt}{0.400pt}}
\multiput(735.00,253.95)(2.695,-0.447){3}{\rule{1.833pt}{0.108pt}}
\multiput(735.00,254.17)(9.195,-3.000){2}{\rule{0.917pt}{0.400pt}}
\multiput(748.00,250.95)(2.695,-0.447){3}{\rule{1.833pt}{0.108pt}}
\multiput(748.00,251.17)(9.195,-3.000){2}{\rule{0.917pt}{0.400pt}}
\multiput(761.00,247.95)(2.918,-0.447){3}{\rule{1.967pt}{0.108pt}}
\multiput(761.00,248.17)(9.918,-3.000){2}{\rule{0.983pt}{0.400pt}}
\multiput(775.00,244.95)(2.695,-0.447){3}{\rule{1.833pt}{0.108pt}}
\multiput(775.00,245.17)(9.195,-3.000){2}{\rule{0.917pt}{0.400pt}}
\put(788,241.17){\rule{2.700pt}{0.400pt}}
\multiput(788.00,242.17)(7.396,-2.000){2}{\rule{1.350pt}{0.400pt}}
\multiput(801.00,239.95)(2.695,-0.447){3}{\rule{1.833pt}{0.108pt}}
\multiput(801.00,240.17)(9.195,-3.000){2}{\rule{0.917pt}{0.400pt}}
\multiput(814.00,236.95)(2.918,-0.447){3}{\rule{1.967pt}{0.108pt}}
\multiput(814.00,237.17)(9.918,-3.000){2}{\rule{0.983pt}{0.400pt}}
\put(828,233.17){\rule{2.700pt}{0.400pt}}
\multiput(828.00,234.17)(7.396,-2.000){2}{\rule{1.350pt}{0.400pt}}
\multiput(841.00,231.95)(2.695,-0.447){3}{\rule{1.833pt}{0.108pt}}
\multiput(841.00,232.17)(9.195,-3.000){2}{\rule{0.917pt}{0.400pt}}
\put(854,228.17){\rule{2.900pt}{0.400pt}}
\multiput(854.00,229.17)(7.981,-2.000){2}{\rule{1.450pt}{0.400pt}}
\multiput(868.00,226.95)(2.695,-0.447){3}{\rule{1.833pt}{0.108pt}}
\multiput(868.00,227.17)(9.195,-3.000){2}{\rule{0.917pt}{0.400pt}}
\put(881,223.17){\rule{2.700pt}{0.400pt}}
\multiput(881.00,224.17)(7.396,-2.000){2}{\rule{1.350pt}{0.400pt}}
\put(894,221.17){\rule{2.700pt}{0.400pt}}
\multiput(894.00,222.17)(7.396,-2.000){2}{\rule{1.350pt}{0.400pt}}
\multiput(907.00,219.95)(2.918,-0.447){3}{\rule{1.967pt}{0.108pt}}
\multiput(907.00,220.17)(9.918,-3.000){2}{\rule{0.983pt}{0.400pt}}
\put(921,216.17){\rule{2.700pt}{0.400pt}}
\multiput(921.00,217.17)(7.396,-2.000){2}{\rule{1.350pt}{0.400pt}}
\put(934,214.17){\rule{2.700pt}{0.400pt}}
\multiput(934.00,215.17)(7.396,-2.000){2}{\rule{1.350pt}{0.400pt}}
\put(947,212.17){\rule{2.900pt}{0.400pt}}
\multiput(947.00,213.17)(7.981,-2.000){2}{\rule{1.450pt}{0.400pt}}
\put(961,210.17){\rule{2.700pt}{0.400pt}}
\multiput(961.00,211.17)(7.396,-2.000){2}{\rule{1.350pt}{0.400pt}}
\multiput(974.00,208.95)(2.695,-0.447){3}{\rule{1.833pt}{0.108pt}}
\multiput(974.00,209.17)(9.195,-3.000){2}{\rule{0.917pt}{0.400pt}}
\put(987,205.17){\rule{2.700pt}{0.400pt}}
\multiput(987.00,206.17)(7.396,-2.000){2}{\rule{1.350pt}{0.400pt}}
\put(1000,203.17){\rule{2.900pt}{0.400pt}}
\multiput(1000.00,204.17)(7.981,-2.000){2}{\rule{1.450pt}{0.400pt}}
\put(1014,201.17){\rule{2.700pt}{0.400pt}}
\multiput(1014.00,202.17)(7.396,-2.000){2}{\rule{1.350pt}{0.400pt}}
\put(1027,199.17){\rule{2.700pt}{0.400pt}}
\multiput(1027.00,200.17)(7.396,-2.000){2}{\rule{1.350pt}{0.400pt}}
\put(1040,197.67){\rule{3.373pt}{0.400pt}}
\multiput(1040.00,198.17)(7.000,-1.000){2}{\rule{1.686pt}{0.400pt}}
\put(1054,196.17){\rule{2.700pt}{0.400pt}}
\multiput(1054.00,197.17)(7.396,-2.000){2}{\rule{1.350pt}{0.400pt}}
\put(1067,194.17){\rule{2.700pt}{0.400pt}}
\multiput(1067.00,195.17)(7.396,-2.000){2}{\rule{1.350pt}{0.400pt}}
\put(1080,192.17){\rule{2.700pt}{0.400pt}}
\multiput(1080.00,193.17)(7.396,-2.000){2}{\rule{1.350pt}{0.400pt}}
\put(1093,190.17){\rule{2.900pt}{0.400pt}}
\multiput(1093.00,191.17)(7.981,-2.000){2}{\rule{1.450pt}{0.400pt}}
\put(1107,188.17){\rule{2.700pt}{0.400pt}}
\multiput(1107.00,189.17)(7.396,-2.000){2}{\rule{1.350pt}{0.400pt}}
\put(1120,186.17){\rule{2.700pt}{0.400pt}}
\multiput(1120.00,187.17)(7.396,-2.000){2}{\rule{1.350pt}{0.400pt}}
\put(1133,184.67){\rule{3.373pt}{0.400pt}}
\multiput(1133.00,185.17)(7.000,-1.000){2}{\rule{1.686pt}{0.400pt}}
\put(1147,183.17){\rule{2.700pt}{0.400pt}}
\multiput(1147.00,184.17)(7.396,-2.000){2}{\rule{1.350pt}{0.400pt}}
\put(1160,181.67){\rule{3.132pt}{0.400pt}}
\multiput(1160.00,182.17)(6.500,-1.000){2}{\rule{1.566pt}{0.400pt}}
\put(1173,180.17){\rule{2.700pt}{0.400pt}}
\multiput(1173.00,181.17)(7.396,-2.000){2}{\rule{1.350pt}{0.400pt}}
\put(1186,178.17){\rule{2.900pt}{0.400pt}}
\multiput(1186.00,179.17)(7.981,-2.000){2}{\rule{1.450pt}{0.400pt}}
\put(1200,176.67){\rule{3.132pt}{0.400pt}}
\multiput(1200.00,177.17)(6.500,-1.000){2}{\rule{1.566pt}{0.400pt}}
\put(1213,175.17){\rule{2.700pt}{0.400pt}}
\multiput(1213.00,176.17)(7.396,-2.000){2}{\rule{1.350pt}{0.400pt}}
\put(1226,173.17){\rule{2.900pt}{0.400pt}}
\multiput(1226.00,174.17)(7.981,-2.000){2}{\rule{1.450pt}{0.400pt}}
\put(1240,171.67){\rule{3.132pt}{0.400pt}}
\multiput(1240.00,172.17)(6.500,-1.000){2}{\rule{1.566pt}{0.400pt}}
\put(1253,170.17){\rule{2.700pt}{0.400pt}}
\multiput(1253.00,171.17)(7.396,-2.000){2}{\rule{1.350pt}{0.400pt}}
\put(1266,168.67){\rule{3.373pt}{0.400pt}}
\multiput(1266.00,169.17)(7.000,-1.000){2}{\rule{1.686pt}{0.400pt}}
\put(1280,167.67){\rule{3.132pt}{0.400pt}}
\multiput(1280.00,168.17)(6.500,-1.000){2}{\rule{1.566pt}{0.400pt}}
\put(1293,166.17){\rule{2.700pt}{0.400pt}}
\multiput(1293.00,167.17)(7.396,-2.000){2}{\rule{1.350pt}{0.400pt}}
\put(1306,164.67){\rule{3.132pt}{0.400pt}}
\multiput(1306.00,165.17)(6.500,-1.000){2}{\rule{1.566pt}{0.400pt}}
\put(1319,163.17){\rule{2.900pt}{0.400pt}}
\multiput(1319.00,164.17)(7.981,-2.000){2}{\rule{1.450pt}{0.400pt}}
\put(1333,161.67){\rule{3.132pt}{0.400pt}}
\multiput(1333.00,162.17)(6.500,-1.000){2}{\rule{1.566pt}{0.400pt}}
\put(1346,160.67){\rule{3.132pt}{0.400pt}}
\multiput(1346.00,161.17)(6.500,-1.000){2}{\rule{1.566pt}{0.400pt}}
\put(1359,159.67){\rule{3.373pt}{0.400pt}}
\multiput(1359.00,160.17)(7.000,-1.000){2}{\rule{1.686pt}{0.400pt}}
\put(1373,158.67){\rule{3.132pt}{0.400pt}}
\multiput(1373.00,159.17)(6.500,-1.000){2}{\rule{1.566pt}{0.400pt}}
\put(1386,157.17){\rule{2.700pt}{0.400pt}}
\multiput(1386.00,158.17)(7.396,-2.000){2}{\rule{1.350pt}{0.400pt}}
\put(1399,155.67){\rule{3.132pt}{0.400pt}}
\multiput(1399.00,156.17)(6.500,-1.000){2}{\rule{1.566pt}{0.400pt}}
\put(1412,154.67){\rule{3.373pt}{0.400pt}}
\multiput(1412.00,155.17)(7.000,-1.000){2}{\rule{1.686pt}{0.400pt}}
\put(1279,694){\makebox(0,0)[r]{$m(x;\eta)=(1-\eta)p(x;\mu_0=0,\sigma_0=1)+\eta p(x;\mu_1=3,\sigma_1=2)$}}
\multiput(1299,694)(20.756,0.000){5}{\usebox{\plotpoint}}
\put(1399,694){\usebox{\plotpoint}}
\put(123,398){\usebox{\plotpoint}}
\put(123.00,398.00){\usebox{\plotpoint}}
\put(141.20,388.06){\usebox{\plotpoint}}
\put(159.75,378.75){\usebox{\plotpoint}}
\put(178.52,369.92){\usebox{\plotpoint}}
\put(197.71,362.03){\usebox{\plotpoint}}
\put(216.73,353.74){\usebox{\plotpoint}}
\put(236.22,346.61){\usebox{\plotpoint}}
\put(255.89,340.03){\usebox{\plotpoint}}
\put(275.46,333.15){\usebox{\plotpoint}}
\put(295.05,326.36){\usebox{\plotpoint}}
\put(314.90,320.31){\usebox{\plotpoint}}
\put(334.79,314.37){\usebox{\plotpoint}}
\put(354.76,308.77){\usebox{\plotpoint}}
\put(374.78,303.38){\usebox{\plotpoint}}
\put(394.86,298.20){\usebox{\plotpoint}}
\put(414.99,293.22){\usebox{\plotpoint}}
\put(435.22,288.56){\usebox{\plotpoint}}
\put(455.27,283.21){\usebox{\plotpoint}}
\put(475.48,278.50){\usebox{\plotpoint}}
\put(495.72,273.90){\usebox{\plotpoint}}
\put(516.15,270.35){\usebox{\plotpoint}}
\put(536.38,265.70){\usebox{\plotpoint}}
\put(556.76,261.81){\usebox{\plotpoint}}
\put(577.09,257.68){\usebox{\plotpoint}}
\put(597.50,253.97){\usebox{\plotpoint}}
\put(617.92,250.33){\usebox{\plotpoint}}
\put(638.31,246.53){\usebox{\plotpoint}}
\put(658.77,243.13){\usebox{\plotpoint}}
\put(679.15,239.28){\usebox{\plotpoint}}
\put(699.69,236.28){\usebox{\plotpoint}}
\put(720.20,233.12){\usebox{\plotpoint}}
\put(740.66,229.69){\usebox{\plotpoint}}
\put(761.07,226.00){\usebox{\plotpoint}}
\put(781.71,223.97){\usebox{\plotpoint}}
\put(802.22,220.81){\usebox{\plotpoint}}
\put(822.75,217.75){\usebox{\plotpoint}}
\put(843.27,214.65){\usebox{\plotpoint}}
\put(863.88,212.29){\usebox{\plotpoint}}
\put(884.43,209.47){\usebox{\plotpoint}}
\put(904.94,206.32){\usebox{\plotpoint}}
\put(925.52,203.65){\usebox{\plotpoint}}
\put(946.11,201.14){\usebox{\plotpoint}}
\put(966.75,199.12){\usebox{\plotpoint}}
\put(987.38,196.94){\usebox{\plotpoint}}
\put(1007.90,193.87){\usebox{\plotpoint}}
\put(1028.54,191.76){\usebox{\plotpoint}}
\put(1049.14,189.35){\usebox{\plotpoint}}
\put(1069.72,186.79){\usebox{\plotpoint}}
\put(1090.42,185.20){\usebox{\plotpoint}}
\put(1111.01,182.69){\usebox{\plotpoint}}
\put(1131.60,180.22){\usebox{\plotpoint}}
\put(1152.29,178.59){\usebox{\plotpoint}}
\put(1172.87,176.02){\usebox{\plotpoint}}
\put(1193.57,174.46){\usebox{\plotpoint}}
\put(1214.25,172.81){\usebox{\plotpoint}}
\put(1234.85,170.37){\usebox{\plotpoint}}
\put(1255.54,168.80){\usebox{\plotpoint}}
\put(1276.24,167.27){\usebox{\plotpoint}}
\put(1296.82,164.71){\usebox{\plotpoint}}
\put(1317.52,163.11){\usebox{\plotpoint}}
\put(1338.22,161.60){\usebox{\plotpoint}}
\put(1358.91,160.01){\usebox{\plotpoint}}
\put(1379.61,158.49){\usebox{\plotpoint}}
\put(1400.31,156.90){\usebox{\plotpoint}}
\put(1421.03,156.00){\usebox{\plotpoint}}
\put(1426,156){\usebox{\plotpoint}}
\put(110.0,82.0){\rule[-0.200pt]{0.400pt}{167.185pt}}
\put(110.0,82.0){\rule[-0.200pt]{320.156pt}{0.400pt}}
\put(1439.0,82.0){\rule[-0.200pt]{0.400pt}{167.185pt}}
\put(110.0,776.0){\rule[-0.200pt]{320.156pt}{0.400pt}}
\end{picture}